\documentclass[twoside]{article}

\usepackage[accepted]{aistats2023}

\usepackage[round]{natbib}

\usepackage{algpseudocode}
\usepackage{algorithm}
\usepackage[usenames,dvipsnames]{xcolor} 
\usepackage{float}
\usepackage{listings}
\usepackage{enumitem}

\usepackage{amsmath,amsfonts,bm}
\usepackage{adjustbox}
\usepackage{booktabs}
\usepackage{hyperref}
\usepackage{subfig}
\def\eqref#1{equation~\ref{#1}}
\def\1{\bm{1}}

\DeclareMathAlphabet{\mathsfit}{\encodingdefault}{\sfdefault}{m}{sl}
\SetMathAlphabet{\mathsfit}{bold}{\encodingdefault}{\sfdefault}{bx}{n}

\newcommand{\softmax}{\mathrm{softmax}}

\usepackage{mathtools}
\usepackage{amsthm, amssymb}

\newtheorem{definition}{Definition}
\newtheorem{lemma}{Lemma}

\newtheorem{theorem}{Theorem}

\usepackage{multirow}

\definecolor{royal}{HTML}{1560BD}
\definecolor{light}{HTML}{1ea2a1}

\usepackage{hyperref}
\hypersetup{
    linkbordercolor=Melon,
    urlbordercolor=magenta,
    citebordercolor=Dandelion,
}

\begin{document}

\twocolumn[

\aistatstitle{Distill n' Explain: explaining graph neural networks using simple surrogates}

\aistatsauthor{ Tamara Pereira$^1$, Erik Nascimento$^1$, Lucas E. Resck$^2$, Diego Mesquita$^2$, Amauri Souza$^{1,3}$}

\aistatsaddress{ $^1$Federal Institute of Cear\'a $^2$Getulio Vargas Foundation $^3$Aalto university } ]

% \aistatsauthor{  Tamara Pereira \And Erik Nascimento \And  Lucas E. Resck \And Diego Mesquita \And Amauri Souza}

% \aistatsaddress{ Federal Institute of Cear\'a \And Federal Institute of Cear\'a  \And  Getulio Vargas Foundation \And Aalto university \And Federal Institute of Cear\'a } ]

\begin{abstract}
Explaining node predictions in graph neural networks (GNNs) often boils down to finding graph substructures that preserve predictions. Finding these structures usually implies back-propagating through the GNN, bonding the complexity (e.g., number of layers) of the GNN to the cost of explaining it. This naturally begs the question: \emph{Can we break this bond by explaining a simpler surrogate GNN?} To answer the question, we propose \emph{Distill n' Explain} (DnX). First, DnX learns a surrogate GNN via \emph{knowledge distillation}. Then, DnX extracts node or edge-level explanations by solving a simple convex program. We also propose FastDnX, a faster version of DnX that leverages the linear decomposition of our surrogate model. Experiments show that DnX and FastDnX often outperform state-of-the-art GNN explainers while being orders of magnitude faster. Additionally, we support our empirical findings with theoretical results linking the quality of the surrogate model (i.e., distillation error) to the faithfulness of explanations. 
\end{abstract}

\section{Introduction}

Graph neural networks (GNNs) \citep{Gori2005,scarselli2009} have become the pillars of representation learning on graphs. Typical GNNs resort to message passing on input graphs to extract meaningful node/graph representations for the task at hand. Despite the success of GNNs in many domains \citep{antibiotic_design,Gilmer2017,recommendersystems,ComplexPhysics}, their architectural design often results in models with limited interpretability. This naturally makes it hard to diagnose scenarios in which GNNs are fooled by confounding effects or align poorly with expert knowledge.

To mitigate this lack of interpretability, a popular strategy is to use post-hoc explanation methods \citep{Ribeiro2016,shap2017,Hima2022, Hima2022b, GraphLIME}. The idea is to increase model transparency by highlighting input/model elements that are particularly important for predictions, helping users to understand what is happening under the hood.

There has been a recent outbreak of methods for explaining GNNs \citep{Yuan2022}. 
Although GNN explanations can come in different flavors \citep{GNNexplainer,subgraphx_icml21,Wang2021,Lucic2022,xgnn_kdd20}, they usually take the form of (minimal) substructures of input graphs that are highly influential to the prediction we want to explain. The seminal work of \citet[GNNExplainer]{GNNexplainer} proposes learning a \emph{soft} mask to weigh graph edges. To find meaningful masks, GNNExplainer maximizes the mutual information between the GNN predictions given the original graph and the masked one. To alleviate the burden of optimizing again whenever we want to explain a different node, \citet[PGExplainer]{PGExplainer} propose using node embeddings to parameterize the masks, i.e., amortizing the inference. 
Nonetheless, GNNExplainer and PGExplainer impose strong assumptions on our access to the GNN we are trying to explain. The former assumes we are able to back-propagate through the GNN. The latter further assumes that we can access hidden activations of the GNN. \citet[PGMExplainer]{PGMExplainer} relieve these assumptions by approximating the local behavior of the GNN with a probabilistic graphical model (PGM) over components, which can be used to rank the relevance of nodes and edges. On the other hand, getting explanations from PGMExplainer involves learning the structure of a PGM, and may not scale well.

In this work, we adopt the same black-box setting of \citet{PGMExplainer} but severely cut down on computational cost by extracting explanations from a \emph{global} surrogate model. In particular, we propose \emph{Distill n' Explain} (DnX). DnX uses knowledge distillation to learn a simple GNN $\Psi$, e.g. simple graph convolution~\citep[SGC]{SGC}, that mimics the behavior of the GNN $\Phi$ we want to explain. 
Then, it solves a simple convex program to find a mask that weighs the influence of each node in the output of $\Psi$.
We also propose FastDnX, a variant of DnX that leverages the linear nature of our surrogate to speed up the explanation procedure.
Notably, we only require evaluations of $\Phi$ to learn the surrogate $\Psi$ and, after $\Psi$ is fixed, we can use it to explain any node-level prediction. To back up the intuition that explaining a surrogate instead of the original GNN is a sensible idea, we provide a theoretical result linking the distillation quality to the faithfulness of our explanations.

Experiments on eight popular node classification benchmarks show that DnX and FastDnX often outperform GNN-, PG-, and PGM-Explainers. We also demonstrate that both DnX and FastDnX are much faster than the competitors. Remarkably, FastDnX presents a speedup of up to $65K\times$ over GNNExplainer. 
Finally, we discuss the limitations of current benchmarks and show that explainers capable of leveraging simple inductive biases can ace them.

\noindent\textbf{Our contributions} are three-fold:
\begin{enumerate}[itemsep=0pt]
    \item we propose a new framework for GNN explanations that treats  GNNs as black-box functions and hinges on explaining a simple surrogate model obtained through knowledge distillation;   
    \item we provide theoretical bounds on the quality of explanations based on these surrogates, linking the error in the distillation procedure to the faithfulness of the explanation;
    \item we carry out extensive experiments, showing that our methods outperform the prior art while running orders of magnitude faster.
\end{enumerate}

\section{Background}
\paragraph{Notation. } We define a graph $\mathcal{G} = (V, E)$, with a set of nodes  $V = \{1, \ldots, n\}$ and a set of edges $E \subseteq V \times V$. We denote the adjacency matrix of $\mathcal{G}$ by $A \in \mathbb{R}^{n\times n}$, i.e., $A_{ij}$ is one if $(i,j) \in E$ and zero otherwise. Let $D$ be the diagonal degree matrix of $\mathcal{G}$, i.e., $D_{i i} = \sum_{j} A_{i j}$. We also define the \emph{normalized} adjacency matrix with added self-loops as $\widetilde{A} = (D + I_n)^{-1/2} (A + I_n) (D + I_n)^{-1/2}$, where $I_n$ is the $n$-dimensional identity matrix. Furthermore, let  $X \in \mathbb{R}^{n \times d}$ be a matrix of $d$-dimensional node features. Throughout this work, we often represent a graph $\mathcal{G}$ using the pair $(A, X)$.

\paragraph{Graph neural networks (GNNs).} We consider the general framework of message-passing GNNs \citep{Gilmer2017}. Typical GNNs interleave aggregation and update steps at each layer. Specifically, for each node $v$ at layer $\ell$, the aggregation is a nonlinear function of the ($\ell-1$)-layer representations of $v$'s neighbors. The update step computes a new representation for $v$ based on its representation at layer $\ell-1$ and the aggregated messages (output of the aggregation step). 
Here we cover two specific GNN architectures: graph convolutional networks \citep[GCNs]{GCN} and simplified graph convolutions \citep[SGC]{SGC}. The former is arguably the most popular GNN in the literature and is used profusely throughout our experiments. The latter is a linear graph model, which will be an asset to our explanation method. 
For a more thorough overview of GNNs, we refer the reader to \citet{book-graph-learning}.

Graph convolutional networks combine local filtering operations (i.e., graph convolutions) and non-linear activation functions (most commonly ReLU) at each layer.
Denoting the weights of the $\ell$-th GCN layer by $W^{(\ell)}$ and the element-wise activation function by $\sigma$, we can recursively write the output of the $\ell$-th layer $H^{(\ell)}$ as:
\begin{equation}
    H^{(\ell)} = \sigma\left(\widetilde{A} H^{(\ell-1)} {W^{(\ell)}}\right),\label{eq:GCN}
\end{equation}
where $H^{(0)} =X$. To obtain node-level predictions, we propagate the final embeddings --- after an arbitrary number of layers --- through a modified convolution with a row-wise softmax instead of $\sigma$, i.e., $\hat{Y} = \softmax(\widetilde{A} H^{(\ell)} W^{(\ell+1)})$. In practice, it is also common to apply multilayer perceptron on top of the final embeddings.

SGC can be viewed as a simplification of the GCN model. \citet{SGC} derive SGC by removing the nonlinear activation functions in GCNs. Consequently, the chained linear transformations become redundant and we can use a single parameter matrix $\Theta$. 
Thus, node predictions from an $L$-layer SGC are:
\begin{align}
\hat{Y} &= \softmax(\widetilde{A}^{L} X \Theta).
\end{align}

Interestingly, \citet{SGC} showed that  SGC often performs similarly to or better than GCN in a variety of node classification tasks. On top of that, training SGCs is computationally more efficient than training GCNs, and SGC has significantly fewer parameters.

\section{DnX: Distill n' Explain}

We now introduce DnX --- a new post-hoc explanation method for GNNs. DnX comprises two steps: knowledge distillation and explanation extraction. During the former, we use a linear GNN $\Psi$ to approximate the predictions from the GNN $\Phi$ we want to explain. 
In the second step, we extract explanations directly from $\Psi$ (instead of $\Phi$). We hypothesize that, as long as $\Psi$ is a good approximation of $\Phi$, substructures highly influential to the output of $\Phi$ should also be relevant to $\Psi$. Therefore, explanations of our surrogate should also explain well the original GNN. To obtain explanations, we exploit the linear nature of $\Psi$ and propose two simple procedures. The first consists of solving a convex program. The second ranks nodes based on a simple decomposition of predictions into additive terms.  

Following \citet{PGMExplainer}, we assume $\Phi$ is a black-box model that we can only probe to get outputs. More specifically, we cannot access gradients of $\Phi$, nor can we access inner layers to extract node embeddings.

\subsection{Knowledge distillation}
\label{sec:kd}

We use SGC~\citep{SGC} to approximate the predictions obtained with the GNN $\Phi$.
Formally, the surrogate model (SGC) $\Psi$ receives the input graph $\mathcal{G}=(A, X)$ and provides class predictions
$\hat{Y}^{(\Psi_\Theta)} = \softmax(\widetilde{A}^L X \Theta)$, where $\Theta$ is the matrix of model parameters, and $L$ is a hyper-parameter. 

The distillation process consists of adjusting the parameters of $\Psi_\Theta$ to match its predictions to those of the network $\Phi$. We do so by minimizing the Kullback-Leibler divergence $\mathrm{KL}$ between the predictions of $\Phi$ and $\Psi_{\Theta}$. 
Let $\hat{Y}_i^{(\Psi_{\Theta})}$ and $\hat{Y}_i^{(\Phi)}$ denote the class predictions for node $i$ from the $\Psi_{\Theta}$ and $\Phi$ models, respectively. We distill $\Phi$ into $\Psi$  by solving: 
\begin{equation}
\min_{\Theta} \left\{  \mathrm{KL} \left( \hat{Y}^{(\Phi)}, \hat{Y}^{(\Psi_\Theta)} \right) \coloneqq
\sum_{i\in V} \sum_{c} \hat{Y}_{i c}^{(\Phi)}  \log \frac{\hat{Y}_{i c}^{(\Phi)}}{\hat{Y}_{i c}^{(\Psi_\Theta)}} \right\},
\label{eq:destilador}
\end{equation}
which is equivalent to  the categorical cross-entropy between $\hat{Y}^{(\Phi)}$ and $\hat{Y}^{(\Psi_\Theta)}$. Note that minimizing this loss does not require back-propagating through the original GNN $\Phi$, only through the surrogate $\Psi$. We also do not require any knowledge about $\Phi$'s architecture.

\subsection{Explanation extraction}
\label{sec:obtaining}

To obtain an explanation to a given prediction $\hat{Y}^{(\Psi_\theta)}_i$, we want to identify a subgraph of $\mathcal{G}$ containing the nodes that influence the most that prediction.
We denote an explanation $\mathcal{E}$ as an $n$-dimensional vector of importance scores (higher equals more relevant), one for each node in the vertex set $V$. 
We introduce two strategies to compute $\mathcal{E}$.

\paragraph{Optimizing for $\mathcal{E}$.} We can formulate the problem of finding the explanation $\mathcal{E}$ by treating it as a vector of 0-1 weights, and minimizing the squared $L_2$ norm between the logits associated with $\hat{Y}_i^{(\Psi_\Theta)}$ and those from the graph with node features masked by $\mathcal{E}$:
\begin{equation}
\min_{\mathcal{E}\in\{0,1\}^n}\parallel \widetilde{A}^{L}_{i}\mathrm{diag}(\mathcal{E})X\Theta - \widetilde{A}^{L}_{i}X\Theta \parallel_{2}^{2},
%- \bm{\zeta}(\bm{X},\bm{A}) \parallel_{2}
\label{eq:e1}
\end{equation}
where $\widetilde{A}^{L}_{i}$ denotes the $i$-th row of the matrix $\widetilde{A}^{L}$. Note that the formulation in \autoref{eq:e1} has a major issue: it admits the trivial solution $\mathcal{E} = [1, 1, \dots, 1]$. To circumvent the issue and simultaneously avoid binary optimization, we replace the search space $\{0, 1\}^n$ by the $(n-1)$-simplex $\Delta = \{r \in \mathbb{R}^n : \sum_i r_i = 1, \forall_i r_i \geq 0 \}$. Implementing this change and re-arranging computations, we wind up with:
\begin{equation}
\min_{\mathcal{E} \in \Delta}\Big\| \widetilde{A}^{L}_{i}\left(\mathrm{diag}(\mathcal{E}) - I_n\right)X\Theta \Big\|_{2}^{2}.
%- \bm{\zeta}(\bm{X},\bm{A}) \parallel_{2}
\label{eq:e2}
\end{equation}

Note that nodes outside the $L$-hop neighborhood of node $i$ do not affect how $\Psi$ classifies it. Thus, we can mask all nodes at distance $\geq L+1$ without altering the solution of \autoref{eq:e2}. For ease of implementation, we solve \autoref{eq:e2} reparameterizing $\mathcal{E}$ as a softmax-transformed vector.
%
% After optimizing for $\bm{\mathcal{E}}$, the values in $\bm{\mathcal{E}}$ serve as a ranking for the importance of each node.

\paragraph{Finding $\mathcal{E}$ via linear decomposition.} Let $Z_i$ denote the logit vector associated with the prediction $\hat{Y}_i^{(\Psi_\Theta)}$. Due to the linear nature of $\Psi$, we can decompose $Z_i$ into a sum of $n$ terms, one for each node in $V$ (plus the bias): 
\begin{equation}
    \widetilde{{A}}^{L}_{i 1} X_1 \Theta+     \widetilde{{A}}^{L}_{i 2} X_2 \Theta + \ldots +     \widetilde{A}^{L}_{i n} X_n\Theta + b=  Z_i. 
\end{equation}
 
Therefore, we can measure the contribution of each node to the prediction as  its scalar projection onto $Z_i - b$:
\begin{equation}
    \mathcal{E}_j \coloneqq \widetilde{{A}}^{L}_{i j} X_j \Theta (Z_i - b)^\intercal
\end{equation}
When we use this strategy instead of solving \autoref{eq:e2}, we refer to our method as FastDnX.

% \newpage
\section{Analysis}
    In this section, we discuss the theoretical and computational aspects of our method.
    We first provide theoretical results supporting the hypothesis that good explanations of a global surrogate $\Psi$ also characterize good explanations of $\Phi$ --- in terms of faithfulness.
    Then, we discuss the convexity of the optimization problem DnX solves to extract explanations.
    We delegate proofs to the \autoref{append:proofs}.

        \label{sec:faith}

        Let $\mathcal{G}_u$ denote the subgraph of $\mathcal{G}$ induced by the $L$-hop neighborhood around node $u$.
        We say an explanation $\mathcal{E}_u$ for a node $u$ is faithful with respect to $\Phi$ if: i) $\Phi$ outputs approximately the same predictions for $u$ regardless of using $\mathcal{E}_u$ to weigh the nodes of $\mathcal{G}_u$ or not; and ii) the same holds under small perturbations of $\mathcal{G}_u$. 
        We can define a perturbation $\mathcal{G}_{u}^\prime$ of $\mathcal{G}_u$  by adding noise to $u$'s features or by randomly rewiring node $u$'s incident edges \citep{agarwal_probing_2022}.
        In this work, we consider perturbations over node features. More precisely, this entails that $V(\mathcal{G}_{u}^\prime) = V(\mathcal{G}_u)$, $E(\mathcal{G}^\prime_u) = E(\mathcal{G}_u)$, and that features are corrupted by noise, i.e., $X^\prime_i = X_i + \epsilon_i$ for $i \in V(\mathcal{G}_u)$ and $\epsilon_i \in \mathbb{R}^{d}$.

        \begin{definition}[Faithfulness]
            \label{def:faithfulness}
            Given a set $\mathcal{K}$ of perturbations of $\mathcal{G}_u$, an explanation $\mathcal{E}_u$ is \textit{faithful} to a model $f$ if
            \[\frac{1}{|\mathcal{K}| + 1} \sum_{ \mathcal{G}_{u}^\prime \in \mathcal{K} \cup \{\mathcal{G}_u\}} \left\lVert f(\mathcal{G}_{u}^\prime)- f(t(\mathcal{G}_{u}^\prime, \mathcal{E}_u))\right\rVert_2 \le \delta,\]
            where $\mathcal{G}_{u}^\prime$ is a possibly perturbed version of $\mathcal{G}_u$, $t$ is a function that applies the explanation $\mathcal{E}_u$ to the graph $\mathcal{G}_{u}^\prime$, and $\delta$ is a small constant \citep{agarwal_probing_2022}.
        \end{definition}

        \autoref{theo:bound_unfaithfulness} provides an upper bound on the  unfaithfulness of $\mathcal{E}_u$ with respect to the surrogate model $\Psi$. \autoref{theo:bound_unfaithfulness_2} extends this result to obtain a bound for $\mathcal{E}_u$ with respect to the model we originally want to explain, i.e., $\Phi$.
    
        \begin{lemma}[Unfaithfulness with respect to $\Psi$]
            \label{theo:bound_unfaithfulness}
            Given a node $u$ and a set $\mathcal{K}$ of perturbations, the unfaithfulness of the explanation $\mathcal{E}_u$ with respect to the prediction $Y_u^{(\Psi_\Theta)}$ of node $u$ is bounded as follows:
            \[\frac{1}{|\mathcal{K}| + 1} \sum_{\substack{ \mathcal{G}^\prime_{u} \in \\ \mathcal{K} \cup \{\mathcal{G}_u\}}} \left\lVert \Psi(\mathcal{G}_{u}^\prime)- \Psi(t(\mathcal{G}_{u}^\prime, \mathcal{E}_u))\right\rVert_2 \le \gamma \left\lVert \underset{\mathcal{E}_u}{\Delta} \widetilde A_u^L \right\rVert_2,\]
            where $\mathcal{G}_{u}^\prime$ is a possibly perturbed version of $\mathcal{G}_u$, $t$ is a function that applies the explanation $\mathcal{E}_u$ to the graph $\mathcal{G}_{u}^\prime$, $\gamma$ is a constant that depends on the model weights $\Theta$, node features $X$, and perturbation $\epsilon$. Furthermore, $\underset{\mathcal{E}_u}{\Delta} \widetilde A_u^L$ is the $u$-th row of the difference of the powered, normalized adjacency matrix $\widetilde A^L$ before and after applying the explanation $\mathcal{E}_u$.
        \end{lemma}
    
        \begin{proof}[Sketch of the proof.]
            We first show that
            \[\left\lVert \Psi(\mathcal{G}_{u})- \Psi(t(\mathcal{G}_{u}, \mathcal{E}_u))\right\rVert_2 \le \lVert (X \Theta)^\intercal \rVert_2 \left\lVert \widetilde A_u^L - \widetilde E_u^L \right\rVert_2\]
            by using Lipschitz continuity of the $\text{softmax}$ function and the compatibility property of the $L_2$ matrix norm.
            We repeat for $\mathcal{G}_u^\prime \in \mathcal{K}$, take the mean in $\mathcal{K} \cup \{\mathcal{G}_u\}$ and isolate $\left\lVert \underset{\mathcal{E}_u}{\Delta} \widetilde A_u^L \right\rVert_2 = \left\lVert \widetilde A_u^L - \widetilde E_u^L \right\rVert_2$.
            The complete proof is available in \autoref{append:proofs}.
        \end{proof}

        \begin{theorem}[Unfaithfulness with respect to $\Phi$]
            \label{theo:bound_unfaithfulness_2}
            Under the same assumptions of \autoref{theo:bound_unfaithfulness} and assuming the $L_2$ distillation error is bounded by $\alpha$, the unfaithfulness of the explanation $\mathcal{E}_u$ for the original model $\Phi$'s node $u$ prediction is bounded as follows:
            \begin{equation*}
                \begin{split}
                    \frac{1}{|\mathcal{K}| + 1} \sum_{\substack{\mathcal{G}_u^\prime \in \\ \mathcal{K} \cup \{\mathcal{G}_u\}}} \left\lVert \Phi(\mathcal{G}_{u}^\prime)- \Phi(t(\mathcal{G}_{u}^\prime, \mathcal{E}_u))\right\rVert_2 \le\ &\gamma \left\lVert \underset{\mathcal{E}_u}{\Delta} \widetilde A_u^L \right\rVert_2 \\
                    &+ 2\alpha.
                \end{split}
            \end{equation*}
        \end{theorem}

        Note that \autoref{theo:bound_unfaithfulness_2} establishes a bound on faithfulness that depends directly on the distillation error $\alpha$. Importantly, when $\Psi$ is a perfect approximation of $\Phi$, we retrieve upper-bound on the RHS of \autoref{theo:bound_unfaithfulness}.

        We note that Theorem 1 by \citet{agarwal_probing_2022} covers an upper bound for the unfaithfulness of GNN explanation methods. However, they do not cover the case in which the explanation is a (weighted) subset of nodes in the $L$-hop neighborhood of $u$, as in our method.

        For completeness, we also extend \autoref{theo:bound_unfaithfulness} and \autoref{theo:bound_unfaithfulness_2} to account for the (very often) probabilistic nature of the noise, i.e., for the case in which $\epsilon_i$ are random variables.
    
        \begin{lemma}[Probability bound on unfaithfulness \emph{w.r.t.} $\Psi$]
            \label{theo:prob_bound_unfaithfulness}
            Given a node $u$ and a set $\mathcal{K}$ of perturbations and assuming the perturbations are i.i.d. with distribution $\epsilon_i~\sim~\mathcal{N}(0, \sigma^2)$, the unfaithfulness of the explanation $\mathcal{E}_u$ with respect to the prediction $Y_u^{(\Psi_\Theta)}$ of node $u$ is bounded in probability as follows:
            \begin{equation*}
                \begin{split}
                    \mathbb{P}\left(
                    \frac{1}{|\mathcal{K}| + 1} \sum_{\substack{ \mathcal{G}^\prime_{u} \in \\ \mathcal{K} \cup \{\mathcal{G}_u\}}} \left\lVert \Psi(\mathcal{G}_{u}^\prime)- \Psi(t(\mathcal{G}_{u}^\prime, \mathcal{E}_u))\right\rVert_2 \le \xi \right) \ge \\
                    \ge F_{\chi_{|\mathcal{K}|nd}^2}\left(\frac{\xi - \gamma_1 \left\lVert \underset{\mathcal{E}_u}{\Delta} \widetilde A_u^L \right\rVert_2}{\gamma_2 \left\lVert \underset{\mathcal{E}_u}{\Delta} \widetilde A_u^L \right\rVert_2 \sigma} - |\mathcal{K}|\right)
                \end{split}
            \end{equation*}
            where $\gamma_1$ is a constant that depends on the model weights $\Theta$ and node features $X$, $\gamma_2$ is a constant that depends on the model weights $\Theta$, and $F_{\chi_{|\mathcal{K}|nd}^2}$ is the c.d.f. of a chi-square r.v. with $|\mathcal{K}| \times n \times d$ degrees of freedom where $(n, d)$ are the row- and column-wise dimensions of $X$.  
        \end{lemma}

        \begin{theorem}[Probability bound on unfaithfulness \emph{w.r.t.} $\Phi$]
            Under the same assumptions of \autoref{theo:prob_bound_unfaithfulness} and assuming the $L_2$ distillation error is bounded by $\alpha$, the unfaithfulness of the explanation $\mathcal{E}_u$ for the original model $\Phi$'s node $u$ prediction is bounded in probability as follows:
            \begin{equation*}
                \begin{split}
                    \mathbb{P}\left(
                    \frac{1}{|\mathcal{K}| + 1} \sum_{\substack{ \mathcal{G}^\prime_{u} \in \\ \mathcal{K} \cup \{\mathcal{G}_u\}}} \left\lVert \Phi(\mathcal{G}_{u}^\prime)- \Phi(t(\mathcal{G}_{u}^\prime, \mathcal{E}_u))\right\rVert_2 \le \xi \right) \ge \\
                    \ge F_{\chi_{|\mathcal{K}|nd}^2}\left(\frac{\xi - \gamma_1 \left\lVert \underset{\mathcal{E}_u}{\Delta} \widetilde A_u^L \right\rVert_2 - 2\alpha}{\gamma_2 \left\lVert \underset{\mathcal{E}_u}{\Delta} \widetilde A_u^L \right\rVert_2 \sigma} - |\mathcal{K}|\right)
                \end{split}
            \end{equation*}
            \label{theo:prob_bound}
        \end{theorem}

In \autoref{theo:prob_bound_unfaithfulness} and \autoref{theo:prob_bound}, when the variance $\sigma^2$ approaches zero, $\xi$ relinquishes its random nature and the probability in the RHS converges to one.
 We note that numerators in the RHS must be non-negative.

Recall DnX/FastDnX's pipeline  involves two steps: model distillation (\autoref{eq:destilador}) and explanation extraction (\autoref{eq:e2}). The former is done only once to learn the surrogate $\Psi$. The latter, however, must be executed for each node whose prediction we want to explain.  Then, gauging the cost of the extraction step may become a genuine concern from a practical point of view, especially for DnX, which implies solving an optimization problem repeatedly. Fortunately, the loss landscape of our extraction problem depends only on the shape of $\Psi$, and not on the original GNN $\Phi$ as in GNNExplainer. Since $\Psi$ is  an SGC, \autoref{eq:e2} is a convex program (\autoref{theo:convexity}) and we reach global optima using, e.g., gradient-based algorithms.
        
        \begin{theorem}[Convexity of DnX]
            \label{theo:convexity}
            The optimization problem of \autoref{eq:e2} is convex.
        \end{theorem}

\section{Additional related works}

\paragraph{Explanations for GNNs.}

The ever-increasing application of GNNs to support high-stake decisions on critical domains \citep{antibiotic_design,Luna2020,Pinion2021} has recently boosted interest in explainability methods for graph models.
\citet{Pope2019} first extended classical gradient-based explanation methods for GNNs. 
Importantly, \citet{GNNexplainer} introduced GNNExplainer and synthetic benchmarks that have been widely adopted to assess GNN explainers. 
Building on parameterized explainers by \citet{PGExplainer}, \citet{Wang2021} proposed ReFine to leverage both global information (e.g., class-wise knowledge) via pre-training and local one (i.e., instance specific patterns) using a fine-tuning process.  
\citet{Lucic2022,Bajaj2021} investigated counterfactual explanations for GNNs, aiming to find minimal perturbations to the input graph such that the prediction changes, e.g., using edge deletions.
\citet{DEGREE} proposed measuring the contribution of different components of the input graph to the GNN prediction by decomposing the information generation and aggregation mechanism of GNNs.
Recently, \citet{Games2022} introduced a structure-aware scoring function derived from cooperative game theory to determine node importance.
Explainability methods for GNNs have also been approached through the lens of causal inference \citep{Lin2021,Lin2022}.
For a more comprehensive coverage of the literature, we refer the reader to \citet{Yuan2022}.

\paragraph{Knowledge distillation.} Since the pivotal work of \citet{Hinton2015}, condensing the knowledge from  a possibly complex \emph{teacher} model into a simpler \emph{student} surrogate has been an active research topic~\citep[e.g.][]{Vadera, Malinin2020Ensemble, ryabinin2021scaling,Ba2022, Clayer}. Nonetheless, despite numerous works using  distillation in image domains~\citep[e.g.][]{Lamp2017, Arthur, Object}, the distillation of GNNs is still a blooming direction. \citet{destillGCN1} proposed the first method for GNN distillation, using a structure-preserving module to explicitly factor in the topological structure embedded by the teacher.
\citep{GCRD} proposed using contrastive learning to implicitly align the node embeddings of the student and the teacher in a common representation space.
\citet{Jing} combined the knowledge of complementary teacher networks into a single student using a dedicated convolutional operator and topological attribution maps.
\citet{Mscale} used an attention mechanism to weigh different teachers depending on the local topology of each node.

\section{Experiments}

In this section, we assess the performance of DnX and FastDnX on several popular benchmarks, including artificial and real-world datasets. We have implemented experiments using PyTorch~\citep{pytorch} and Torch Geometric~\citep{torch_geometric}. Our code is available at \url{https://github.com/tamararruda/DnX}.

\subsection{Experimental setup}

\paragraph{Datasets.} We consider six synthetic datasets broadly used for evaluating explanations of GNNs: BA-House-Shapes, BA-Community, BA-Grids, Tree-Cycles,  Tree-Grids, and BA-Bottle-Shaped. These datasets are available in \citep{GNNexplainer} and \citep{PGMExplainer}. 
Each dataset is a single graph with multiple copies of identical motifs connected to base subgraphs. These subgraphs either consists of random sample graphs from the Barabási–Albert (BA) model \citep{Barabasi1999} or 8-level balanced binary trees.
An explanation associated with a motif-node must only include motif elements.
Thus, base nodes denote information irrelevant to the prediction of any node.

We also use two real-world datasets: Bitcoin-Alpha and Bitcoin-OTC~\citep{bitcoin-otc-alpha2016,bitcoin-otc-alpha2018}. These datasets denote networks in which nodes correspond to user accounts that trade Bitcoin. A directed edge $(u,v)$ (between users $u$ and $v$) denotes the degree of reliability assigned by $u$ to $v$, i.e., each edge has a score denoting the degree of trust.
\autoref{append:implementation} provides more details regarding datasets. 

\paragraph{Baselines.}

We compare DnX against three baseline explainers: GNNExplainer \citep{GNNexplainer}, PGExplainer \citep{PGExplainer}, and PGMExplainer \citep{PGMExplainer}.
To ensure a valid comparison, we closely follow guidelines and the evaluation setup from the original works.
We first generate explanations for a 3-layer GCN~\citep{GCN} with ReLU activation. We also consider three additional architectures: graph isomorphism networks (GIN) \citep{xu2018gin}, gated graph sequence neural networks (GATED) \citep{li2015gated} and auto-regressive moving average GNNs (ARMA) \citep{ARMA}
This allows for evaluating the robustness and performance of explainers across GNNs of different complexities.

\begin{table*}[ht]
\centering
\caption{Performance (accuracy) of explanation methods for node-level explanations (i.e., explanations given as subsets of nodes) in the synthetic datasets. \textcolor{royal}{Blue} and \textcolor{light}{Green} numbers denote the best and second-best methods, respectively. Standard deviations are taken over 10 runs of the explanation process, distillation is not included. Since FastDnX's explanations are deterministic, we mark its variance with not applicable (NA). In most cases, FastDnX achieves the best performance.
}
\adjustbox{width=\textwidth}{
\begin{tabular}{llccccccc}
\toprule
%  & syn 1  & syn 2  & syn 3 & syn 4 & syn 5 &  syn 6 \\ \toprule
\textbf{Model} & \textbf{Explainer}  &\textbf{BA-House} & \textbf{BA-Community}  & \textbf{BA-Grids} & \textbf{Tree-Cycles} & \textbf{Tree-Grids} & \textbf{ BA-Bottle} \\ \toprule
%GI

\multirow{5}{*}{GCN} & GNNExplainer & $77.5\pm{1.2}$ & $64.7\pm{1.0}$ & $89.2\pm{2.0}$ & $77.2\pm{9.0}$ & $71.1\pm{1.0}$ & $73.3\pm{3.0}$ \\
 & PGExplainer & $95.0\pm{1.1}$ & $70.6\pm{2.0}$ & $86.2\pm{9.0}$ & \textcolor{light}{$\bm{92.4\pm{5.2}}$} & $76.7\pm{1.2}$ & $98.2\pm{3.0}$  \\
  & PGMExplainer & \textcolor{light}{$\bm{97.9\pm{0.9}}$} & $92.2\pm{0.2}$ & $88.6\pm{0.9}$ & \textcolor{royal}{$\bm{94.1\pm{0.8}}$} & \textcolor{royal}{$\bm{86.8\pm{2.0}}$} & $97.5\pm{1.5}$ \\
 \cmidrule{2-8}

 & DnX &  $97.7\pm{0.2}$ & \textcolor{light}{$\bm{94.6\pm{0.1}}$} &  \textcolor{light}{$\bm{89.8\pm{0.1}}$} & $83.3\pm{0.4}$ & $80.2\pm{0.1}$ & \textcolor{light}{$\bm{99.6 \pm 0.1}$}\\
 & FastDnX & \textcolor{royal}{$\bm{99.6 \pm \text{NA}}$} & \textcolor{royal}{$\bm{95.4\pm \text{NA}}$}& \textcolor{royal}{$\bm{93.9\pm \text{NA}}$} & $87.3\pm \text{NA}$ & \textcolor{light}{$\bm{85.0 \pm \text{NA}}$} & \textcolor{royal}{$\bm{99.8 \pm \text{NA}}$}\\ \midrule\midrule

 %ARMA
\multirow{5}{*}{ARMA} & GNNExplainer & $80.9\pm{1.2}$  & $78.5\pm{1.0}$ &  $87.3\pm{1.3}$ & $77.7\pm{1.0}$ & $79.3\pm{1.1}$ &$84.3\pm{1.3}$  \\
 & PGExplainer & $91.4\pm{0.1}$  & $72.1\pm{0.1}$ &  $83.8\pm{1.0}$ & \textcolor{light}{$\bm{92.6\pm{2.1}}$} & $85.1\pm{0.1}$ &$97.0\pm{1.1}$  \\
  &PGMExplainer & \textcolor{light}{$\bm{99.3\pm{0.2}}$} & $67.5\pm{0.8}$ & $86.8\pm{0.3}$ & \textcolor{royal}{$\bm{95.0\pm{0.2}}$}& \textcolor{royal}{$\bm{90.6\pm{0.3}}$} & \textcolor{light}{$\bm{99.7\pm{0.1}}$} \\
 \cmidrule{2-8}
 & DnX & $98.1\pm{0.2}$  & \textcolor{light}{$\bm{92.7\pm{0.2}}$} & \textcolor{light}{$\bm{90.8\pm{0.1}}$} &$ 83.5\pm{0.4}$ & $79.6\pm{0.3}$ & $96.9\pm{0.2}$\\
 & FastDnX & \textcolor{royal}{$\bm{100.0 \pm \text{NA}}$} & \textcolor{royal}{$\bm{95.2\pm \text{NA}}$} & \textcolor{royal}{$\bm{94.7\pm \text{NA}}$} & $87.1\pm \text{NA}$ & \textcolor{light}{$\bm{87.7\pm \text{NA}}$} & \textcolor{royal}{$\bm{99.9 \pm \text{NA}}$} \\\midrule \midrule

%GATED
 \multirow{5}{*}{GATED} & GNNExplainer & $79.7\pm{1.0}$  & $68.8\pm{1.0}$ &  \textcolor{light}{$\bm{91.4\pm{3.0}}$} & $85.2\pm{2.0}$ & $73.2\pm{4.0}$ &$70.0\pm{2.0}$  \\
 & PGExplainer & $96.1\pm{4.1}$  & $70.9\pm{3.0}$ &  $90.7\pm{1.0}$ & \textcolor{light}{$\bm{91.7\pm{7.0}}$} & $83.7\pm{1.5}$ &\textcolor{royal}{$\bm{98.7\pm{0.1}}$}  \\
  &PGMExplainer & \textcolor{light}{$\bm{98.6\pm{0.1}}$} & $69.4\pm{0.5}$ & $86.8\pm{0.3}$ & \textcolor{royal}{$\bm{94.1\pm{0.2}}$}& \textcolor{royal}{$\bm{90.1\pm{0.2}}$} &\textcolor{light}{$\bm{98.3\pm{0.2}}$} \\
  \cmidrule{2-8}

 & DnX & $98.3\pm{0.1}$  &\textcolor{light}{ $\bm{91.1\pm{0.1}}$ }&$ 90.8\pm{0.1}$ & $85.0\pm{0.3}$ & $82.1\pm{0.2}$ &$98.0\pm{0.2}$\\
 & FastDnX & \textcolor{royal}{$\bm{99.6\pm{\text{NA}}}$} & \textcolor{royal}{$\bm{93.5\pm{\text{NA}}}$} & \textcolor{royal}{$\bm{94.0\pm{\text{NA}}}$} & $76.8\pm{\text{NA}}$ & \textcolor{light}{$\bm{86.8\pm{\text{NA}}}$} & $98.0\pm{\text{NA}}$ \\ \midrule \midrule
 
 %GIN
% & GNNExplainer &  &  &   &  &  &  \\
% & PGExplainer &  &  &   &  &  &   \\
\multirow{3}{*}{GIN}  & PGMExplainer & $60.2\pm{0.2}$ &  $84.5\pm{0.3}$ & $68.4\pm{0.2}$ &  \textcolor{royal}{$\bm{89.3\pm{0.2}}$}& \textcolor{royal}{$\bm{85.0\pm{0.5}}$} & $55.7\pm{0.4}$ \\
  \cmidrule{2-8}
  & DnX &  \textcolor{light}{$\bm{99.0\pm{0.1}}$} & \textcolor{light}{$\bm{94.0\pm{0.2}}$} & \textcolor{light}{$\bm{91.1\pm{0.1}}$} & \textcolor{light}{$\bm{84.1\pm{0.3}}$} & \textcolor{light}{$\bm{77.3\pm{0.2}}$} & \textcolor{light}{$\bm{95.3\pm{0.2}}$}\\
 & FastDnX & \textcolor{royal}{$\bm{99.6\pm{\text{NA}}}$} & \textcolor{royal}{$\bm{94.7\pm{\text{NA}}}$}& \textcolor{royal}{$\bm{93.9\pm{\text{NA}}}$} &  $75.2\pm{\text{NA}}$ & $76.5\pm{\text{NA}}$ &\textcolor{royal}{ $\bm{99.1\pm{\text{NA}}}$}\\

\bottomrule
\end{tabular}}
%\end{adjustbox}
\label{tab:result_exp_node}
\end{table*}

\paragraph{Implementation details.} We use an 80/10/10\% (train/val/test) split for all datasets. All GNNs have 3 layers and are trained for $1000$ epochs, with early stopping if the validation accuracy does not improve in $100$ consecutive epochs.  We train all baseline GNNs using Adam~\citep{adam} with a learning rate of 0.01 with a weight decay of $5.0 \times 10^{-4}$.
We show the performance of these GNNs on the benchmark datasets in the supplementary material.
Importantly, we observe accuracy $\geq 95\%$ for most data/model combinations.

For the distillation phase in DnX, we use an SGC model with $3$ layers. We use the predictions for all nodes to train the surrogate SGC. For the optimization, we use AdamW~\citep{AdamW} with a learning rate of $0.1$ with a weight decay of $5.0 \times 10^{-6}$ and $10000$ epochs.

It is worth mentioning that PGExplainer and GNNExplainer --- as described in the experimental section of their respective papers --- output edge-level explanations, so their results are not immediately comparable to that of our methods and PGMExplainer. More specifically, the two former output importance scores for each edge. On the other hand, our methods and PGMExplainer output node importance scores.
Therefore, we convert edge-level explanations to node-level ones by averaging over the scores of all edges incident in a node. For completeness, we  provide additional results doing the reverse transformation (i.e., node- to edge-level explanations) in  the Supplement.

\subsection{Results} \autoref{tab:result_exp_node} compares the performance of DnX and FastDnX against previous art in terms of explanation accuracy, i.e., the number of nodes in method's output that are also in the ground-truth explanations divided by the total number of nodes in the latter. Overall, FastDnX is the best-performing method for all network architectures (GCN, ARMA, GATED, and GIN) on all datasets but Tree-Cycles and Tree-Grids. For Tree-Grids, FastDnX places second for GCN, ARMA and GATED whereas PGMExplainer obtains the highest accuracies. We also note that, while DnX is often better than GNNExplainer and PGExplainer, its performance bests FastDnX only in $12.5\%$ of cases. GNN- and PGExplainer do not appear in the comparison for GIN since they require propagating edge masks, and Torch Geometric does not support edge features for GIN.

\autoref{tab:result_exp_bitcoin_node} reports the performance of all explainers on the Bitcoin-Alpha and Bitcoin-OTC datasets. Following previous work \citep{PGMExplainer}, we use average precision (AP) as evaluation metric, i.e., the percentage of top-$k$ nodes obtained from each explainer that are correct, averaged over all nodes to be explained. While running the experiments, we noticed that the evaluation protocol employed by \citet{PGMExplainer} obtains explanations for a 3-layer GCN but only considers 1-hop candidate nodes during the explanation phase. This implies that some potentially relevant  nodes are discarded by design. \autoref{tab:result_exp_bitcoin_node} shows results for both 1-hop and 3-hop settings. DnX is the best-performing method, and its fast variant is the second-best across all experiments. For 3-hop candidate nodes, the absolute precision gap between DnX and the best baseline is at least 14\% for Bitcoin-Alpha and 11\% for Bitcoin-OTC. Overall, DnX outperforms GNNExplainer and PGMExplainer by a large margin. Note that the performance of PGMExplainer drops considerably when going from 1- to 3-hop. We report additional results in the Appendix.

\begin{table}[!thb]
\centering
\caption{Performance (average precision) of node-level explanations for real-world datasets. \textcolor{royal}{Blue} and \textcolor{light}{Green} numbers denote the best and second-best methods, respectively. DnX significantly outperforms the baselines (GNN-, PG-, and PGM-Explainers).}
\adjustbox{width=0.48\textwidth}{
\begin{tabular}{cl|ccc|ccc}
\toprule
&\multicolumn{1}{c}{} &\multicolumn{3}{c}{\textbf{Bitcoin-Alpha}} & \multicolumn{3}{c}{\textbf{Bitcoin-OTC}}\\
\midrule

\textbf{GNN} & \textbf{Explainer} &    \textbf{top 3} & \textbf{top 4} &\textbf{ top 5}  &  \textbf{top 3} & \textbf{top 4} & \textbf{top 5}\\ \midrule

\multirow{5}{*}{\shortstack{GCN \\ (1-hop)}} & GNNEx  &  $86.3$ & $85.2$ & $81.2$  & $83.3 $  &$ 81.7$ & $ 77.0$\\
& PGEx   & $83.5$ &  $83.6$ & $79.5$ &  $ 79.9$ & $ 80.1$ & $ 76.6$ \\
 & PGMEx  & $87.3$ & $85.7$ & $84.8$ & $83.3$ & $81.7$ &  $80.8$\\

\cmidrule{2-8}
& DnX  & \textcolor{royal}{$\bm{92.2}$}  & \textcolor{royal}{$\bm{89.5}$} & \textcolor{royal}{$\bm{88.4}$} & \textcolor{royal}{$\bm{89.4}$} & \textcolor{royal}{$\bm{86.6}$}  & \textcolor{royal}{$\bm{84.7}$}\\
& FastDnX  & \textcolor{light}{$\bm{89.4}$} & \textcolor{light}{$\bm{87.8}$} & \textcolor{light}{$\bm{86.8}$} & \textcolor{light}{$\bm{87.7}$} & \textcolor{light}{$\bm{85.1}$} & \textcolor{light}{$\bm{83.4}$} \\

\midrule \midrule 

\multirow{5}{*}{\shortstack{GCN \\ (3-hop)}} & GNNEx  &  $80.1$ & $74.9$ & $70.9$ & $82.4$ & $79.6$ & $70.6$ \\
%& PGEx   & $0.005$ & $0.005$ & $0.005$ & $0.004$ & $0.004$ & $0.004$\\
& PGEx   & $81.5$ & $78.1 $ & $ 69.5$ & $ 78.5$ & $74.5 $ & $ 67.4$\\
& PGMEx  & $67.0$ & $59.8$ & $51.8$ & $63.0$ & $55.2$ &  $47.4$\\
\cmidrule{2-8}

& DnX  & \textcolor{royal}{$\bm{95.8}$} & \textcolor{royal}{$\bm{91.9}$} &  \textcolor{royal}{$\bm{87.9}$} & \textcolor{royal}{$\bm{94.8}$} & \textcolor{royal}{$\bm{91.4}$} & \textcolor{royal}{$\bm{86.3}$} \\
& FastDnX  & \textcolor{light}{$\bm{89.8}$} & \textcolor{light}{$\bm{85.2}$} & \textcolor{light}{$\bm{80.2}$} & \textcolor{light}{$\bm{88.0}$} & \textcolor{light}{$\bm{83.0}$} & \textcolor{light}{$\bm{78.8}$} \\

\bottomrule
\end{tabular}}

\label{tab:result_exp_bitcoin_node}
\end{table}

\begin{figure*}[t]
    \centering
    \includegraphics[width=\textwidth]{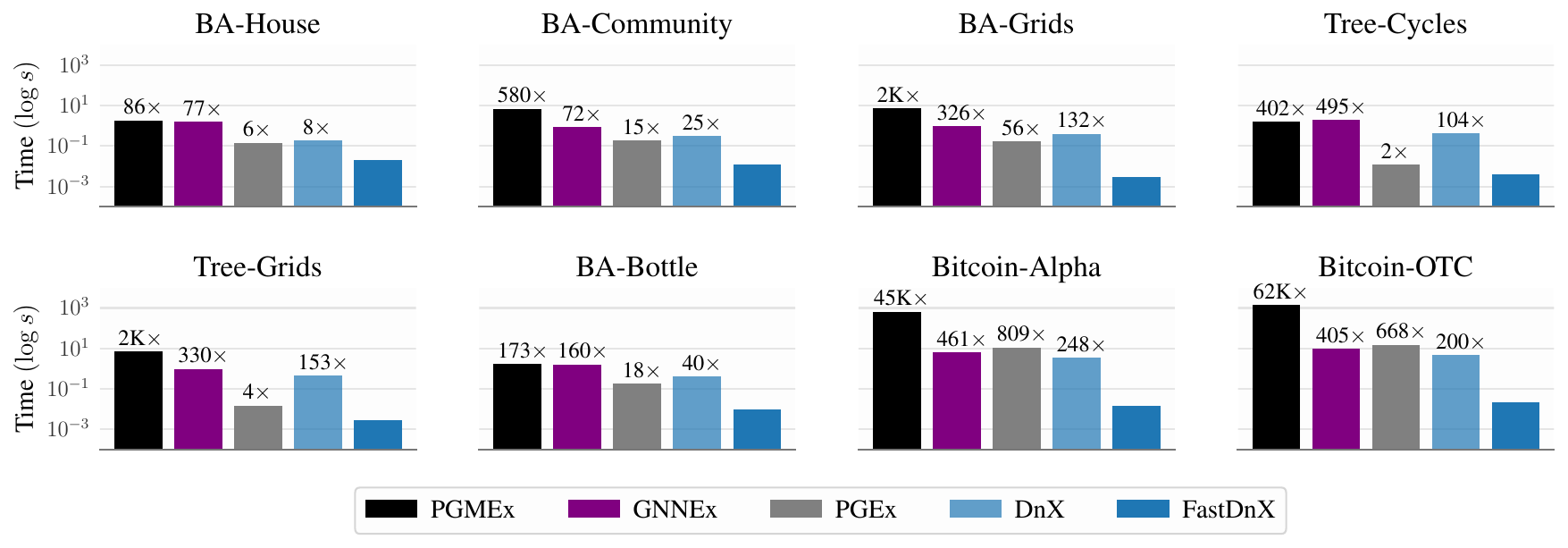}
    \caption{Time comparison. The bar plots show the average time each method takes to explain a prediction from GCN. FastDnX is consistently the fastest method, often by a large margin. For the datasets with largest average degree (Bitcoin datasets), FastDnX is 4 orders of magnitude faster than PGMExplainer and 2 orders faster than the other methods.}
    \label{fig:times}
\end{figure*}

\paragraph{Time comparison.} To demonstrate the computational efficiency of DnX/FastDnX,
\autoref{fig:times} shows the time each method takes to explain a single GCN prediction.
For a fair comparison, we also take into account the distillation step in DnX/FastDnX. 
In particular, we add a fraction -- one over the total number of nodes we wish to explain -- of the distillation time and add it to the time DnX and FastDnX actually take to generate an explanation.
Notably, both DnX and FastDnX are consistently much faster than GNNExplainer and PGMExplainer.
For instance, FastDnX is more than forty thousand times faster than PGMExplainer in Bitcoin-Alpha and Bitcoin-OTC.

\paragraph{Distillation results.} For completeness, \autoref{tab:distillation_gcn} shows the distillation accuracy achieved by our linear network $\Psi$ when $\Phi$ is a GCN,  for both the synthetic and the real datasets. Here, we measure accuracy using the predictions of the model $\Phi$ as ground truth. For all cases, we observe accuracy superior to $86\%$. \autoref{tab:distillation_gcn} also shows the time elapsed during the distillation step.
Similar results are achieved when distilling ARMA, GATED and GIN models, these results are shown and described in the Appendix.
%\textcolor{red}{\bf Appendix ? shows similar results for ARMA, GATED, and GIN.}

\begin{table}[!htb]
\centering
\caption{Distillation accuracy and time for GCN. For all cases, accuracy $>86\%$ and the distillation phase takes considerably less than 1 minute.} 
% \adjustbox{width=\columnwidth}{
\begin{tabular}{lcc}
\toprule
\textbf{Dataset} & \textbf{Accuracy} & \textbf{Time (s)} \\
\midrule
BA-House & $94.2 \pm{1.2}$ & $13.996$ \\
BA-Community  & $86.6\pm{0.1}$ & $16.447$ \\
BA-Grids & $99.9\pm{0.1}$ & $2.721$ \\
Tree-Cycles & $97.7\pm{0.2}$ & $3.820$\\
Tree-Grids & $98.0\pm{0.2}$ & $3.803$\\
BA-Bottle & $98.5\pm{0.2}$ & $3.181$\\
Bitcoin-Alpha & $90.4\pm{0.1}$ & $28.317$\\
Bitcoin-OTC & $89.1\pm{0.2}$ & $32.414$\\
\bottomrule
%
% &\multicolumn{7}{c}{Average Time (s)} \\ \hline
% 13.996 & 16.347 & 2.721 & 3.820 & 3.803 & 3.181 & 28.317 & 32.414 \\
\end{tabular}
% }
\label{tab:distillation_gcn}
\end{table}

Interestingly, although BA-community is the dataset with the lowest distillation accuracy (86.6\%), DnX and FastDnX achieve significantly better results than the previous state-of-the-art (cf. \autoref{tab:result_exp_node}). The rationale for these counter-intuitive results is that the distiller can differentiate between motif nodes and base nodes, and this is enough to get good explanations -- since  the evaluation set comprises motif nodes only. More concretely, the confusion matrix in \autoref{fig:conf_mat} reveals that, despite the low distillation accuracy, the surrogate model $\Psi$ correctly predicts the base nodes (classes 1 and 5). Therefore, $\Psi$ achieves high accuracy for the binary classification problem of distinguishing motif and base nodes, supporting our hypothesis.

\begin{figure}[thb]
\centering
\includegraphics[width=0.8\columnwidth]{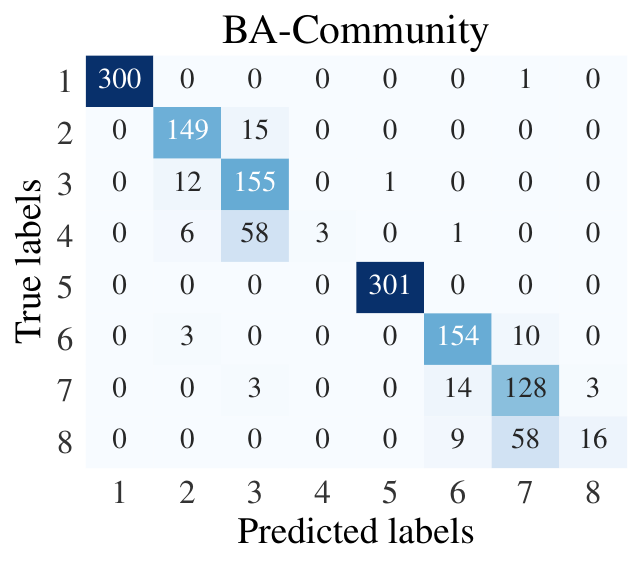}
\caption{Confusion matrix of the distillation process for the BA-Community dataset. Classes 1 and 5 correspond to base nodes. While the surrogate misclassifies many motif nodes, it is able to correctly predict almost all base ones.}
\label{fig:conf_mat}
\end{figure}

\paragraph{Fidelity results.}

\begin{table*}[ht]
\centering
\caption{Performance (fidelity) of different methods for node-level explanations (i.e., explanations given as subsets of nodes) on synthetic datasets. The numbers in \textcolor{royal}{Blue} and \textcolor{light}{Green} denote the best and second-best methods, respectively.  The closer to zero, the better. DnX performs as well as or better than GNNEx, PGEx, and PGMEx in 5 out of 6 datasets.}
%\adjustbox{width=\textwidth}
{
\begin{tabular}{lccccccc}
\toprule
%  & syn 1  & syn 2  & syn 3 & syn 4 & syn 5 &  syn 6 \\ \toprule
\textbf{Explainer}  &\textbf{BA-House} & \textbf{BA-Community}  & \textbf{BA-Grids} & \textbf{Tree-Cycles} & \textbf{Tree-Grids} & \textbf{ BA-Bottle} \\ \toprule
%GI

GNNEx    & $0.035$         & $-0.276$        &  \textcolor{light}{$\bm{0.015}$ }        & $-0.810$         & $ -0.120$      & $ -0.290$ \\
PGEx     & $0.035$         & \textcolor{royal}{$\bm{-0.232}$}        & $-0.194$        & $-0.830 $       & $-0.175$      & $0.142 $  \\ 
PGMEx    & $0.035$         & $-0.290$        & \textcolor{light}{$\bm{0.015}$}         & $-0.677$        & $-0.005 $     &\textcolor{light}{$\bm{ 0.025} $}   \\ \midrule
DnX      & $0.035$         & $-0.286$        & \textcolor{royal}{$\bm{0.008}$}       & \textcolor{royal}{$\bm{-0.230}$}      & \textcolor{light}{$\bm{-0.001}$}     & \textcolor{royal}{$\bm{0.002}$}    \\
FastDnX  & $0.035$         & \textcolor{light}{$\bm{-0.272}$}       &$-0.018 $       & \textcolor{light}{$\bm{-0.240}$}      & \textcolor{royal}{$\bm{0.000}$}       & $0.050$    \\
 % -  & $0.035$         & ${-0.286}$       &$0.015 $       & ${-0.442}$      & ${0.000}$       & $0.005$    \\

\bottomrule
\end{tabular}}
\label{tab:result_exp_node_fidelity}
\end{table*}

\begin{table*}[ht]
\centering
\caption{Performance (fidelity) of different methods for node-level explanations on real-world datasets. The numbers in \textcolor{royal}{Blue} and  \textcolor{light}{Green} denote the best and second-best method, respectively. The closer to zero the better.  We show results for sparsity levels of 30\%, 50\% e 70\%. In all cases, FastDnX or DnX are among the two best-performing methods.}
\adjustbox{width=0.725\textwidth}
{
\begin{tabular}{cl|c|c|c|c|c}
\toprule
%&\multicolumn{1}{c}{} &\multicolumn{1}{c}{Bitcoin-Alpha} & \multicolumn{1}{c}{Bitcoin-OTC}\\
\textbf{Sparsity} & \textbf{Explainer} &  \textbf{Bitcoin-Alpha}  &  \textbf{Bitcoin-OTC} &  \textbf{Cora} &  \textbf{Citeseer} &  \textbf{Pubmed} \\ \hline
\multirow{5}{*}{\shortstack{ $30 \% $}}

& GNNex  & \textcolor{royal}{$\bm{0.008}$} & ${0.060}$ & $0.015$ &\textcolor{royal}{$ \bm{0.006}$}   &  \textcolor{royal}{$\bm{0.000 }$} \\
& PGEx   & $0.101$ & $0.100$ & $0.019$ & $0.051$   & $0.046$ \\
& PGMEx  & $0.154$ & $0.155$ &  \textcolor{light}{$\bm{0.013}$} & $0.012 $  & -      \\
& DnX    &  $0.028$  &  \textcolor{royal}{$\bm{0.020}$}   &  \textcolor{royal}{$\bm{0.007}$} & \textcolor{royal}{$\bm{0.006}$}   & \textcolor{light}{$\bm{0.015}$}\\
& FastDnX  &  \textcolor{light}{$\bm{0.012}$} &  \textcolor{light}{$\bm{0.036}$} & $0.015$ &  \textcolor{royal}{$\bm{0.006}$} & \textcolor{light}{$\bm{0.015}$} \\ 
% & -  & $0.202$ & $0.174$ &$0.033$  & $ 0.048$   & $0.051$ \\
\midrule

\multirow{5}{*}{\shortstack{ $50 \% $}}

& GNNex  &  $0.148$ & $0.040$ & $0.014$ & \textcolor{royal}{$\bm{0.003}$}& \textcolor{light}{$\bm{0.006}$}\\
& PGEx   & $0.102$ & $0.107$ & $0.014$ & $0.027$   & $0.025$ \\
& PGMEx  & $0.102$ & $0.118$& $0.011$ & \textcolor{royal}{$\bm{-0.003}$}  & -  \\
& DnX  & \textcolor{light}{$\bm{0.012}$} &\textcolor{royal}{$\bm{0.018}$}& \textcolor{royal}{$\bm{0.000}$}  & $0.009$ &$ 0.010$\\
& FastDnX  & \textcolor{royal}{$\bm{0.004}$} & \textcolor{light}{$\bm{ 0.056 }$}& \textcolor{royal}{$\bm{ 0.000} $}& \textcolor{royal}{$\bm{0.003}$} &\textcolor{royal}{$\bm{0.005}$}\\
% & -  & $0.130$ & $0.098$ &$0.056$  & $0.018$   & $0.045$ \\
\midrule
\multirow{5}{*}{\shortstack{ $70 \% $}}

& GNNex  & \textcolor{light}{$\bm{-0.004}$} & $0.016 $ & $0.015 $ & $-0.009$   & \textcolor{light}{$\bm{-0.005}$}\\
& PGEx   & $0.091 $&$ 0.099$ & \textcolor{light}{$\bm{-0.003}$} & ${-0.006}$ &  \textcolor{light}{$\bm{0.005}$} \\
& PGMEx  & $0.088$ &$ 0.099$  & $0.009$  & $0.008$   & -      \\
& DnX  & \textcolor{royal}{$\bm{0.000}$}& \textcolor{royal}{$\bm{0.000}$} & $-0.004$ & \textcolor{royal}{$\bm{0.003}$}   & \textcolor{royal}{$\bm{0.000}$}\\
& FastDnX  & \textcolor{light}{$\bm{0.004}$} & \textcolor{light}{$\bm{-0.012}$} & \textcolor{royal}{$\bm{0.000}$}  & \textcolor{royal}{$\bm{-0.003}$}   & $0.010$ \\
% & -  & $0.214$ & $0.074$ &$0.026$  & $0.021$   & $0.015$ \\

\bottomrule
\end{tabular}}
\label{tab:result_real_fidelity}
\end{table*}

To further assess the quality of explanations, we consider a fidelity metric --- we use \emph{Fidelity-} as in \citep{Yuan2022}. This metric measures how the GNN's predictive performance (accuracy) fluctuates when we classify nodes based only on the subgraph induced by the explanations. When the fidelity is positive, there is a decrease in performance. When it is negative, using ``only the explanation'' yields better predictions on average. 
Tables \ref{tab:result_exp_node_fidelity} and \ref{tab:result_real_fidelity} report fidelity for the synthetic and the real datasets, respectively.
Note that we have considered three additional real-world datasets (citation networks): Cora, Citeseer, and Pubmed. 
Results obtained from DnX for the synthetic datasets are the best ones in 50\% of the cases.
It is interesting to observe that for Tree-Cycles and Tree-Grids, DnX/FastDnX are not the best performing ones wrt accuracy (Table 1), but are the best ones wrt fidelity (Table 4). 
For real datasets, in most cases, either DnX or FastDnX achieves the best results overall. 
Importantly, this corroborates the results we observed for the precision metric on Bitcoin-Alpha/OTC datasets. 
We note that it was infeasible to run PGMExplainer on Pubmed as explaining one prediction with it can take up to an hour in our local hardware.

\section{Discussion}\label{sec:discussion}
\vspace{-12pt}
\paragraph{Are benchmarks too simple?} Given that DnX/FastDnX often achieve remarkable performance by explaining simple surrogates, a natural questions arises: \emph{are these popular benchmarks for GNN explanations too simple?} Since these benchmarks rely on model-agnostic ground-truth explanations, we now investigate inductive biases behind these explanations, and show that they can be easily captured. 

\autoref{fig:degree} reports the degree distribution of motif and base nodes for all synthetic datasets. Recall that, by design, ground-truth explanations are always given by motif nodes. Note also that support for the distributions motif and base nodes have almost no overlap for most datasets (except Tree-Cycles \& Tree-Grids). Thus, any explainer capable of leveraging degree information would obtain high accuracy.  

To make this more concrete, we propose a very simple baseline "explainer" that outputs explanations based on the normalized adjacency matrix. 
In particular, we define the importance of node $j$ to the prediction of node $i$ as the $(i,j)$-entry of $\widetilde{A}^{L}$, with $L=3$. With this simple baseline, we obtain the following accuracy values: 99.9\% (BA-House), 98.1\% (BA-Community), 99.9\% (BA-Grids), 95.9\% (Tree-Cycles), 90.4\% (Tree-Grids), and 99.9\% (BA-Bottle). Notably, this baseline would rank 1 if included as 
an explanation method for GCNs in Table 1.

\citet{Faber2021} have also raised issues regarding these benchmarks, proposing alternative datasets as well. We have run FastDnX to explain a 2-layer GCN model for two of their proposed datasets (\emph{Community} and \emph{Negative evidence}), and obtained remarkably good accuracy results: 94.0\% and 99.5\%, respectively. Also, simply ranking nodes based on the entries of $\widetilde{A}^{L}$ ($L=2$) achieves accuracy of 93.0\% (\emph{Community}) and 99.6\% (\emph{Neg. evidence}).

\begin{figure}[thb]
\centering
\includegraphics[width=0.88\columnwidth]{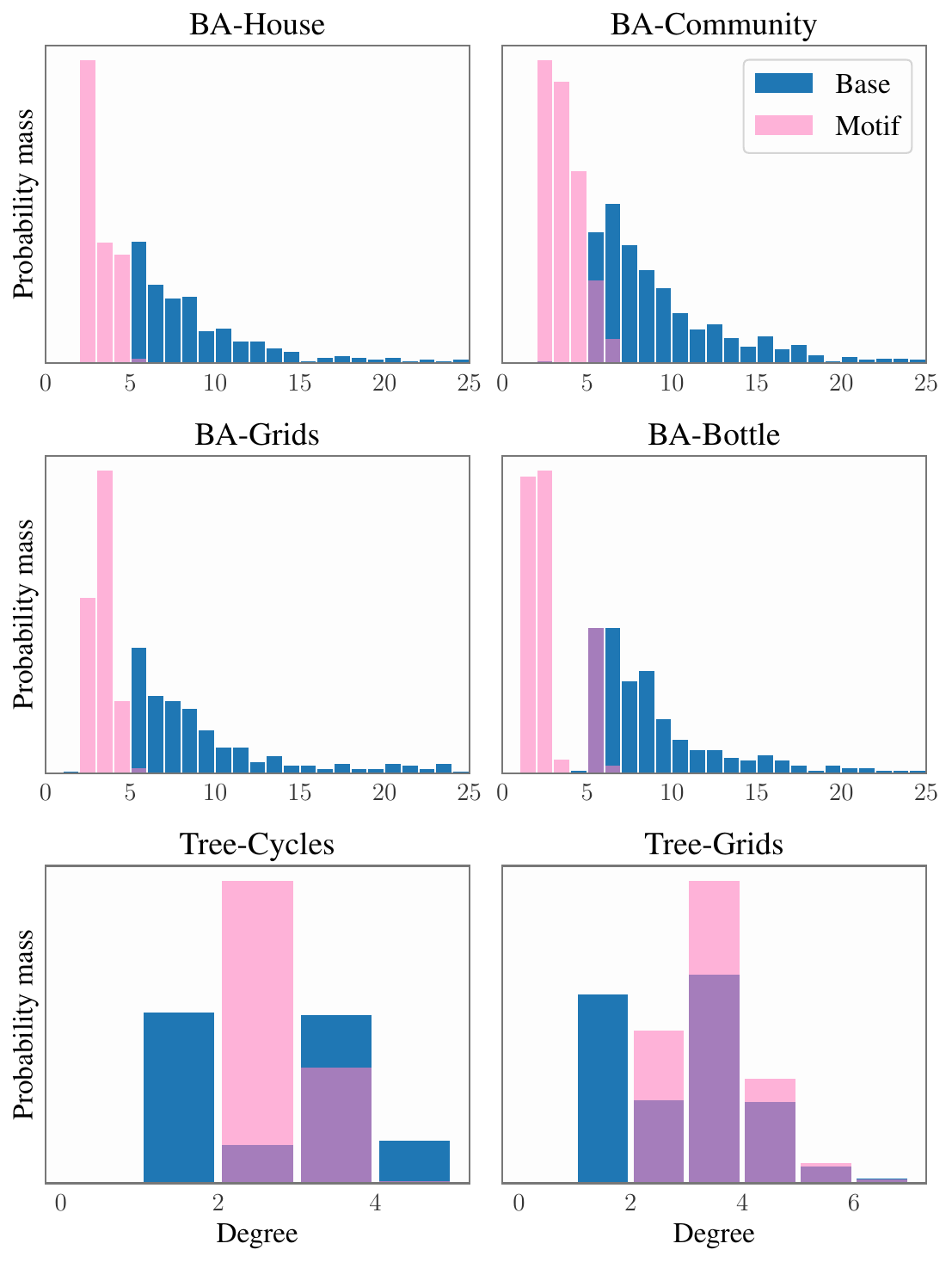}
%  \centering
%  \adjustbox{width=0.5\textwidth}{
%      \subfloat{\includegraphics[width=\textwidth]{images/degree/BA-House-degree.png}}
%      \subfloat{\includegraphics[width=\textwidth]{images/degree/BA-Community-degree.png}}}\\
% \adjustbox{width=0.5\textwidth}{
%     \subfloat{\includegraphics[width=\textwidth]{images/degree/BA-Grids-degree.png}}
%     \subfloat{\includegraphics[width=\textwidth]{images/degree/BA-Bottle-degree.png}}}\\
% \adjustbox{width=0.5\textwidth}{    \subfloat{\includegraphics[width=\textwidth]{images/degree/Tree-Cycles-degree.png}}
%     \subfloat{\includegraphics[width=\textwidth]{images/degree/Tree-Grids-degree.png}}}
 \caption{Degree distribution of motif and base nodes. While we can overall distinguish motif and base nodes from degree information on BA-based datasets, there is a significant overlap on Tree-Cycles and Tree-Grids.}
 \label{fig:degree}

\end{figure}

\paragraph{Limitations.} While simple graph models (like SGC) have been shown to achieve good performance on node-level classification tasks, they fail to rival recent GNNs for graph-level prediction tasks \citep{CS, SGC}. Naturally, we would not expect DnX and FastDnX to work well out-of-the-shelf to explain graph-level predictions.  However, our methods could be easily extended to use more powerful linear GNNs that incorporate different types of diffusion operators~\citep{sign}, or use long-range residual connections \citep{SimpleDeep}.

\section{Conclusion}
This work proposes \emph{DnX} as a simple and intuitive two-step framework for post-hoc explanation of GNNs. First, we distill the GNN into a simpler and more interpretable one, that serves as a global surrogate. Then, we leverage the simple structure of the surrogate to extract explanations. Experiments show that (Fast)DnX outperforms the prior art on a variety of benchmarks. Remarkably, our simple design allows FastDnX to run at least $200\times$ faster than relevant baselines on real-world tasks. 
Additionally, we provide theoretical results that justify our framework and support our empirical findings.
Besides advancing the current art, we hope this work will motivate other researchers to focus on developing compute-efficient explainability methods.

\section*{Acknowledgments}

This work was supported by the Silicon Valley Community Foundation (SVCF) through the Ripple impact fund, the Funda\c{c}\~ao de Amparo \`a Pesquisa do Estado do Rio de Janeiro (FAPERJ), the Funda\c{c}\~ao Cearense de Apoio ao Desenvolvimento Científico e Tecnológico (FUNCAP), the Coordena\c{c}\~ao de Aperfei\c{c}oamento de Pessoal de N\'ivel Superior (CAPES), and the Getulio Vargas Foundation's school of applied mathematics (FGV EMAp).

 \newpage

\bibliographystyle{plainnat}
\bibliography{references}

\begin{thebibliography}{60}
\providecommand{\natexlab}[1]{#1}
\providecommand{\url}[1]{\texttt{#1}}
\expandafter\ifx\csname urlstyle\endcsname\relax
  \providecommand{\doi}[1]{doi: #1}\else
  \providecommand{\doi}{doi: \begingroup \urlstyle{rm}\Url}\fi

\bibitem[Agarwal et~al.(2022)Agarwal, Zitnik, and
  Lakkaraju]{agarwal_probing_2022}
C.~Agarwal, M.~Zitnik, and H.~Lakkaraju.
\newblock Probing {GNN} {Explainers}: {A} {Rigorous} {Theoretical} and
  {Empirical} {Analysis} of {GNN} {Explanation} {Methods}.
\newblock In \emph{International Conference on Artificial Intelligence and
  Statistics (AISTATS)}, 2022.

\bibitem[Baek et~al.(2022)Baek, Oh, Lee, Lee, and Ham]{Object}
D.~Baek, Y.~Oh, S.~Lee, J.~Lee, and B.~Ham.
\newblock Decomposed knowledge distillation for class-incremental semantic
  segmentation.
\newblock In \emph{Advances in Neural Information Processing Systems
  (NeurIPS)}, 2022.

\bibitem[Bajaj et~al.(2021)Bajaj, Chu, Xue, Pei, Wang, Lam, and
  Zhang]{Bajaj2021}
M.~Bajaj, L.~Chu, Z.~Y. Xue, J.~Pei, L.~Wang, P.~C.-H. Lam, and Y.~Zhang.
\newblock Robust counterfactual explanations on graph neural networks.
\newblock In \emph{Advances in Neural Information Processing Systems
  (NeurIPS)}, 2021.

\bibitem[Barabási and Albert(1999)]{Barabasi1999}
A.~Barabási and R.~Albert.
\newblock Emergence of scaling in random networks.
\newblock \emph{Science}, 286\penalty0 (5439), 1999.

\bibitem[Bianchi et~al.(2022)Bianchi, Grattarola, Livi, and Alippi]{ARMA}
F.~Maria Bianchi, D.~Grattarola, L.~Livi, and C.~Alippi.
\newblock Graph neural networks with convolutional arma filters.
\newblock \emph{IEEE Transactions on Pattern Analysis and Machine
  Intelligence}, 44\penalty0 (7), 2022.

\bibitem[Chen et~al.(2020)Chen, Wei, Huang, Ding, and Li]{SimpleDeep}
M.~Chen, Z.~Wei, Z.~Huang, B.~Ding, and Y~Li.
\newblock Simple and deep graph convolutional networks.
\newblock In \emph{International Conference on Machine Learning (ICML)}, 2020.

\bibitem[Derrow-Pinion et~al.(2021)Derrow-Pinion, She, Wong, Lange, Hester,
  Perez, Nunkesser, Lee, Guo, Wiltshire, Battaglia, Gupta, Li, Xu,
  Sanchez-Gonzalez, Li, and Velickovic]{Pinion2021}
A.~Derrow-Pinion, J.~She, D.~Wong, O.~Lange, T.~Hester, L.~Perez, M.~Nunkesser,
  S.~Lee, X.~Guo, B.~Wiltshire, P.~W. Battaglia, V.~Gupta, A.~Li, Z.~Xu,
  A.~Sanchez-Gonzalez, Y.~Li, and P.~Velickovic.
\newblock Eta prediction with graph neural networks in google maps.
\newblock In \emph{Conference on Information and Knowledge Management (CIKM)},
  2021.

\bibitem[Douillard et~al.(2021)Douillard, Chen, Dapogny, and Cord]{Arthur}
A.~Douillard, Y.~Chen, A.~Dapogny, and M.~Cord.
\newblock {PLOP}: Learning without forgetting for continual semantic
  segmentation.
\newblock In \emph{IEEE/CVF Conference on Computer Vision and Pattern
  Recognition (CVPR)}, 2021.

\bibitem[Faber et~al.(2021)Faber, Moghaddam, and Wattenhofer]{Faber2021}
L.~Faber, A.~K. Moghaddam, and R.~Wattenhofer.
\newblock When comparing to ground truth is wrong: On evaluating gnn
  explanation methods.
\newblock In \emph{Conference on Knowledge Discovery \& Data Mining (KDD)},
  2021.

\bibitem[Feng et~al.(2021)Feng, Liu, Yang, Tang, Du, and Hu]{DEGREE}
Q.~Feng, N.~Liu, F.~Yang, R.~Tang, M.~Du, and X.~Hu.
\newblock Degree: Decomposition based explanation for graph neural networks.
\newblock In \emph{International Conference on Learning Representations
  (ICLR)}, 2021.

\bibitem[Fey and Lenssen(2019)]{torch_geometric}
M.~Fey and J.~E. Lenssen.
\newblock Fast graph representation learning with {PyTorch Geometric}.
\newblock In \emph{Workshop on Representation Learning on Graphs and Manifolds
  (ICLR)}, 2019.

\bibitem[Gao and Pavel(2018)]{gao_properties_2018}
B.~Gao and L.~Pavel.
\newblock On the {Properties} of the {Softmax} {Function} with {Application} in
  {Game} {Theory} and {Reinforcement} {Learning}.
\newblock \emph{ArXiv:1704.00805}, 2018.

\bibitem[Gilmer et~al.(2017)Gilmer, Schoenholz, Riley, Vinyals, and
  Dahl]{Gilmer2017}
J.~Gilmer, S.~S. Schoenholz, P.~F. Riley, O.~Vinyals, and G.~E. Dahl.
\newblock Neural message passing for quantum chemistry.
\newblock In \emph{International Conference on Machine Learning (ICML)}, 2017.

\bibitem[Gori et~al.(2005)Gori, Monfardini, and Scarselli]{Gori2005}
M.~Gori, G.~Monfardini, and F.~Scarselli.
\newblock A new model for learning in graph domains.
\newblock In \emph{IEEE International Joint Conference on Neural Networks
  (IJCNN)}, 2005.

\bibitem[Hamilton(2020)]{book-graph-learning}
W.~L. Hamilton.
\newblock Graph representation learning.
\newblock \emph{Synthesis Lectures on Artificial Intelligence and Machine
  Learning}, 14\penalty0 (3):\penalty0 1--159, 2020.

\bibitem[Han et~al.(2022)Han, Srinivas, and Lakkaraju]{Hima2022b}
T.~Han, S.~Srinivas, and H.~Lakkaraju.
\newblock Which explanation should i choose? a function approximation
  perspective to characterizing post hoc explanations.
\newblock \emph{Advances in Neural Information Processing Systems (NeurIPS)},
  2022.

\bibitem[Hen et~al.(2021)Hen, Mei, Zhang, Wang, Wang, Feng, and Chen]{Clayer}
D.~Hen, J.~Mei, Y.~Zhang, C.~Wang, Z.~Wang, Y.~Feng, and C.~Chen.
\newblock Cross-layer distillation with semantic calibration.
\newblock In \emph{AAAI Conference on Artificial Intelligence (AAAI)}, 2021.

\bibitem[Hinton et~al.(2015)Hinton, Vinyals, and Dean]{Hinton2015}
G.~E. Hinton, O~Vinyals, and J.~Dean.
\newblock Distilling the knowledge in a neural network.
\newblock \emph{Arxiv:1503.02531}, 2015.

\bibitem[Huang et~al.(2021)Huang, He, Singh, Lim, and Benson]{CS}
Q.~Huang, H.~He, A.~Singh, S.~Lim, and A.~Benson.
\newblock Combining label propagation and simple models out-performs graph
  neural networks.
\newblock In \emph{International Conference on Learning Representations
  (ICLR)}, 2021.

\bibitem[Huang et~al.(2022)Huang, Yamada, Tian, Singh, and Chang]{GraphLIME}
Q.~Huang, M.~Yamada, Y.~Tian, D.~Singh, and Y.~Chang.
\newblock Graphlime: Local interpretable model explanations for graph neural
  networks.
\newblock \emph{IEEE Transactions on Knowledge and Data Engineering}, 2022.

\bibitem[Jim{\'e}nez-Luna et~al.(2020)Jim{\'e}nez-Luna, Grisoni, and
  Schneider]{Luna2020}
J.~Jim{\'e}nez-Luna, F.~Grisoni, and G.~Schneider.
\newblock Drug discovery with explainable artificial intelligence.
\newblock \emph{Nature Machine Intelligence}, 2\penalty0 (10):\penalty0
  573--584, 2020.

\bibitem[Jing et~al.(2021)Jing, Yang, Wang, Song, and Tao]{Jing}
Y.~Jing, Y.~Yang, X.~Wang, M.~Song, and D.~Tao.
\newblock Amalgamating knowledge from heterogeneous graph neural networks.
\newblock In \emph{IEEE/CVF Conference on Computer Vision and Pattern
  Recognition (CVPR)}, 2021.

\bibitem[Joshi et~al.(2021)Joshi, Liu, Xun, Lin, and Foo]{GCRD}
C.~Joshi, F.~Liu, X.~Xun, J.~Lin, and C.~Foo.
\newblock On representation knowledge distillation for graph neural networks.
\newblock \emph{arXiv:2111.04964}, 2021.

\bibitem[Kingma and Ba(2015)]{adam}
D.~Kingma and J.~Ba.
\newblock Adam: A method for stochastic optimization.
\newblock In \emph{International Conference on Learning Representations
  (ICLR)}, 2015.

\bibitem[Kipf and Welling(2017)]{GCN}
T.~N. Kipf and M.~Welling.
\newblock Semi-supervised classification with graph convolutional networks.
\newblock In \emph{International Conference on Learning Representations
  (ICLR)}, 2017.

\bibitem[Kumar et~al.(2016)Kumar, Spezzano, Subrahmanian, and
  Faloutsos]{bitcoin-otc-alpha2016}
S.~Kumar, F.~Spezzano, V.~Subrahmanian, and C.~Faloutsos.
\newblock Edge weight prediction in weighted signed networks.
\newblock In \emph{International Conference on Data Mining (ICDM)}, 2016.

\bibitem[Kumar et~al.(2018)Kumar, Hooi, Makhija, Kumar, Faloutsos, and
  Subrahmanian]{bitcoin-otc-alpha2018}
S.~Kumar, B.~Hooi, D.~Makhija, M.~Kumar, C.~Faloutsos, and V.~Subrahmanian.
\newblock Rev2: Fraudulent user prediction in rating platforms.
\newblock In \emph{International Conference on Web Search and Data Mining
  (WSDM)}, 2018.

\bibitem[Li et~al.(2016)Li, Tarlow, Brockschmidt, and Zemel]{li2015gated}
Y.~Li, D.~Tarlow, M.~Brockschmidt, and R.~Zemel.
\newblock Gated graph sequence neural networks.
\newblock \emph{International Conference on Learning Representations (ICLR)},
  2016.

\bibitem[Lin et~al.(2021)Lin, Lan, and Li]{Lin2021}
W.~Lin, H.~Lan, and B.~Li.
\newblock Generative causal explanations for graph neural networks.
\newblock In \emph{International Conference on Machine Learning (ICML)}, 2021.

\bibitem[Lin et~al.(2022)Lin, Lan, Wang, and Li]{Lin2022}
W.~Lin, H.~Lan, H.~Wang, and B.~Li.
\newblock Orphicx: A causality-inspired latent variable model for interpreting
  graph neural networks.
\newblock In \emph{IEEE/CVF Conference on Computer Vision and Pattern
  Recognition (CVPR)}, 2022.

\bibitem[Loshchilov and Hutter(2019)]{AdamW}
I.~Loshchilov and F.~Hutter.
\newblock Decoupled weight decay regularization.
\newblock In \emph{International Conference on Learning Representations
  (ICLR)}, 2019.

\bibitem[Lucic et~al.(2022)Lucic, Ter~Hoeve, Tolomei, De~Rijke, and
  Silvestri]{Lucic2022}
A.~Lucic, M.~Ter~Hoeve, G.~Tolomei, M.~De~Rijke, and F.~Silvestri.
\newblock Cf-gnnexplainer: Counterfactual explanations for graph neural
  networks.
\newblock In \emph{International Conference on Artificial Intelligence and
  Statistics (AISTATS)}, 2022.

\bibitem[Lundberg and Lee(2017)]{shap2017}
S.~M. Lundberg and S.-I. Lee.
\newblock A unified approach to interpreting model predictions.
\newblock In \emph{Advances in Neural Information Processing Systems
  (NeurIPS)}, 2017.

\bibitem[Luo et~al.(2020)Luo, Cheng, Xu, Yu, Zong, Chen, and
  Zhang]{PGExplainer}
D.~Luo, W.~Cheng, D.~Xu, W.~Yu, B.~Zong, H.~Chen, and X.~Zhang.
\newblock Parameterized explainer for graph neural network.
\newblock In \emph{Advances in Neural Information Processing Systems
  (NeurIPS)}, 2020.

\bibitem[Malinin et~al.(2020)Malinin, Mlodozeniec, and
  Gales]{Malinin2020Ensemble}
A.~Malinin, B.~Mlodozeniec, and M.~Gales.
\newblock Ensemble distribution distillation.
\newblock In \emph{International Conference on Learning Representations
  (ICLR)}, 2020.

\bibitem[Paszke et~al.(2017)Paszke, Gross, Chintala, Chanan, Yang, DeVito, Lin,
  Desmaison, Antiga, and Lerer]{pytorch}
A.~Paszke, S.~Gross, S.~Chintala, G.~Chanan, E.~Yang, Z.~DeVito, Z.~Lin,
  A.~Desmaison, L.~Antiga, and A.~Lerer.
\newblock Automatic differentiation in pytorch.
\newblock In \emph{Advances in Neural Information Processing Systems (NeurIPS -
  Workshop)}, 2017.

\bibitem[Pope et~al.(2019)Pope, Kolouri, Rostami, Martin, and
  Hoffmann]{Pope2019}
P.~E. Pope, S.~Kolouri, M.~Rostami, C.~E. Martin, and H.~Hoffmann.
\newblock Explainability methods for graph convolutional neural networks.
\newblock In \emph{IEEE/CVF Conference on Computer Vision and Pattern
  Recognition (CVPR)}, 2019.

\bibitem[Rebuffi et~al.(2017)Rebuffi, Kolesnikov, Sperl, and Lampert]{Lamp2017}
S.-A. Rebuffi, A.~Kolesnikov, G.~Sperl, and C.~Lampert.
\newblock i{C}a{RL}: Incremental classifier and representation learning.
\newblock In \emph{IEEE/CVF Conference on Computer Vision and Pattern
  Recognition (CVPR)}, 2017.

\bibitem[Ribeiro et~al.(2016)Ribeiro, Singh, and Guestrin]{Ribeiro2016}
M.~T. Ribeiro, S.~Singh, and C.~Guestrin.
\newblock "{W}hy should {I} trust you?": Explaining the predictions of any
  classifier.
\newblock In \emph{International Conference on Knowledge Discovery and Data
  Mining (KDD)}, 2016.

\bibitem[Rossi et~al.(2020)Rossi, Frasca, Chamberlain, Eynard, Bronstein, and
  Monti]{sign}
E.~Rossi, F.~Frasca, B.~Chamberlain, D.~Eynard, M.~M. Bronstein, and F.~Monti.
\newblock {SIGN:} scalable inception graph neural networks.
\newblock \emph{Workshop on Graph Representation Learning and Beyond (ICML)},
  2020.

\bibitem[Ryabinin et~al.(2021)Ryabinin, Malinin, and
  Gales]{ryabinin2021scaling}
M.~Ryabinin, A.~Malinin, and M.~Gales.
\newblock Scaling ensemble distribution distillation to many classes with proxy
  targets.
\newblock In \emph{Advances in Neural Information Processing Systems
  (NeurIPS)}, 2021.

\bibitem[Sanchez-Gonzalez et~al.(2020)Sanchez-Gonzalez, Godwin, Pfaff, Ying,
  Leskovec, and Battaglia]{ComplexPhysics}
A.~Sanchez-Gonzalez, J.~Godwin, T.~Pfaff, R.~Ying, J.~Leskovec, and
  P.~Battaglia.
\newblock Learning to simulate complex physics with graph networks.
\newblock In \emph{International Conference on Machine Learning (ICML)}, 2020.

\bibitem[Scarselli et~al.(2009)Scarselli, Gori, Tsoi, Hagenbuchner, and
  Monfardini]{scarselli2009}
F.~Scarselli, M.~Gori, A.~C. Tsoi, M.~Hagenbuchner, and G.~Monfardini.
\newblock The graph neural network model.
\newblock \emph{IEEE Transactions on Neural Networks}, 20\penalty0 (1), 2009.

\bibitem[Slack et~al.(2021)Slack, Hilgard, Singh, and Lakkaraju]{Hima2022}
D.~Slack, A.~Hilgard, S.~Singh, and H.~Lakkaraju.
\newblock Reliable post hoc explanations: Modeling uncertainty in
  explainability.
\newblock In \emph{Advances in Neural Information Processing Systems
  (NeurIPS)}, 2021.

\bibitem[Stokes et~al.(2020)Stokes, Yang, Swanson, Jin, Cubillos-Ruiz, Donghia,
  MacNair, French, Carfrae, Bloom-Ackermann, Tran, Chiappino-Pepe, Badran,
  Andrews, Chory, Church, Brown, Jaakkola, Barzilay, and
  Collins]{antibiotic_design}
J.~M. Stokes, K.~Yang, K.~Swanson, W.~Jin, A.~Cubillos-Ruiz, N.~M. Donghia,
  C.~R. MacNair, S.~French, L.~A. Carfrae, Z.~Bloom-Ackermann, V.~M. Tran,
  A.~Chiappino-Pepe, A.~H. Badran, I.~W. Andrews, E.~J. Chory, G.~M. Church,
  E.~D. Brown, T.~S. Jaakkola, R.~Barzilay, and J.~J. Collins.
\newblock A deep learning approach to antibiotic discovery.
\newblock \emph{Cell}, 180\penalty0 (4), 2020.

\bibitem[Vadera et~al.(2020)Vadera, Jalaian, and Marlin]{Vadera}
M.~Vadera, B.~Jalaian, and B.~Marlin.
\newblock Generalized bayesian posterior expectation distillation for deep
  neural networks.
\newblock In \emph{Uncertainty in Artificial Intelligence (UAI)}, 2020.

\bibitem[van~de Geijn and Myers(2022)]{van_de_geijn_advanced_2022}
R.~van~de Geijn and M.~Myers.
\newblock \emph{Advanced Linear Algebra Foundations to Frontiers}.
\newblock open EdX Publisher, Austin, Texas, 2022.

\bibitem[Vu and Thai(2020)]{PGMExplainer}
M.~Vu and M.~T. Thai.
\newblock Pgm-explainer: Probabilistic graphical model explanations for graph
  neural networks.
\newblock In \emph{Advances in Neural Information Processing Systems
  (NeurIPS)}, 2020.

\bibitem[Wang et~al.(2021)Wang, Wu, Zhang, He, and Chua]{Wang2021}
X.~Wang, Y.~Wu, A.~Zhang, X.~He, and T.~Chua.
\newblock Towards multi-grained explainability for graph neural networks.
\newblock In \emph{Advances in Neural Information Processing Systems
  (NeurIPS)}, 2021.

\bibitem[Wu et~al.(2019)Wu, Souza, Zhang, Fifty, Yu, and Weinberger]{SGC}
F.~Wu, A.~Souza, T.~Zhang, C.~Fifty, T.~Yu, and K.~Weinberger.
\newblock Simplifying graph convolutional networks.
\newblock In \emph{International Conference on Machine Learning (ICML)}, 2019.

\bibitem[Xu et~al.(2019)Xu, Hu, Leskovec, and Jegelka]{xu2018gin}
K.~Xu, W.~Hu, J.~Leskovec, and S.~Jegelka.
\newblock How powerful are graph neural networks?
\newblock \emph{International Conference on Learning Representations (ICLR)},
  2019.

\bibitem[Yang et~al.(2020)Yang, Qiu, Song, Tao, and Wang]{destillGCN1}
Y.~Yang, J.~Qiu, M.~Song, D.~Tao, and X.~Wang.
\newblock Distilling knowledge from graph convolutional networks.
\newblock In \emph{IEEE/CVF Conference on Computer Vision and Pattern
  Recognition (CVPR)}, pages 7074--7083, 2020.

\bibitem[Ying et~al.(2018)Ying, He, Chen, Eksombatchai, Hamilton, and
  Leskovec]{recommendersystems}
R.~Ying, R.~He, K.~Chen, P.~Eksombatchai, W.~L. Hamilton, and J.~Leskovec.
\newblock Graph convolutional neural networks for web-scale recommender
  systems.
\newblock In \emph{International Conference on Knowledge Discovery \& Data
  Mining (KDD)}, 2018.

\bibitem[Ying et~al.(2019)Ying, Bourgeois, You, Zitnik, and
  Leskovec]{GNNexplainer}
Z.~Ying, D.~Bourgeois, J.~You, M.~Zitnik, and J~Leskovec.
\newblock Gnnexplainer: Generating explanations for graph neural networks.
\newblock In H.~Wallach, H.~Larochelle, A.~Beygelzimer, F.~d\textquotesingle
  Alch\'{e}-Buc, E.~Fox, and R.~Garnett, editors, \emph{Advances in Neural
  Information Processing Systems (NeurIPS)}, 2019.

\bibitem[Yuan et~al.(2020)Yuan, Tang, Hu, and Ji]{xgnn_kdd20}
H.~Yuan, J.~Tang, X.~Hu, and S.~Ji.
\newblock {XGNN}: Towards model-level explanations of graph neural networks.
\newblock In \emph{International Conference on Knowledge Discovery \& Data
  Mining (KDD)}, 2020.

\bibitem[Yuan et~al.(2021)Yuan, Yu, Wang, Li, and Ji]{subgraphx_icml21}
H.~Yuan, H.~Yu, J.~Wang, K.~Li, and S.~Ji.
\newblock On explainability of graph neural networks via subgraph explorations.
\newblock In \emph{International Conference on Machine Learning (ICML)}, 2021.

\bibitem[Yuan et~al.(2022)Yuan, Yu, Gui, and Ji]{Yuan2022}
H.~Yuan, H.~Yu, S.~Gui, and S.~Ji.
\newblock Explainability in graph neural networks: A taxonomic survey.
\newblock \emph{IEEE transactions on pattern analysis and machine
  intelligence}, 2022.

\bibitem[Zhang et~al.(2022{\natexlab{a}})Zhang, Liu, Dang, and Zhang]{Mscale}
C.~Zhang, J.~Liu, K.~Dang, and W.~Zhang.
\newblock Multi-scale distillation from multiple graph neural networks.
\newblock In \emph{AAAI Conference on Artificial Intelligence (AAAI)},
  2022{\natexlab{a}}.

\bibitem[Zhang et~al.(2022{\natexlab{b}})Zhang, Shah, Liu, and Sun]{Games2022}
S.~Zhang, N.~Shah, Y.~Liu, and Y.~Sun.
\newblock Explaining graph neural networks with structure-aware cooperative
  games.
\newblock In \emph{Advances in Neural Information Processing Systems
  (NeurIPS)}, 2022{\natexlab{b}}.

\bibitem[Zhou et~al.(2022)Zhou, Nezhadarya, and Ba]{Ba2022}
Y.~Zhou, E.~Nezhadarya, and J.~Ba.
\newblock Dataset distillation using neural feature regression.
\newblock In \emph{Advances in Neural Information Processing Systems
  (NeurIPS)}, 2022.

\end{thebibliography}

\appendix
\clearpage
\newpage

\section{Proofs}
    \label{append:proofs}
    
    \begin{proof}[Proof of Lemma~\ref{theo:bound_unfaithfulness}]
        For node $u$, it is known, by definition, that
        \begin{equation*}
            \begin{split}
                &\left\lVert \Psi(\mathcal{G}_{u})- \Psi(t(\mathcal{G}_{u}, \mathcal{E}_u))\right\rVert_2 \\
                =\ & \left\lVert \sigma\left(\widetilde A^L X \Theta\right)_u - \sigma\left(\widetilde E^L X \Theta\right)_u \right\rVert_2 \\
                =\ &\left\lVert \sigma\left[\left(\widetilde A^L X \Theta\right)_u\right] - \sigma\left[\left(\widetilde E^L X \Theta\right)_u\right] \right\rVert_2,
            \end{split}
        \end{equation*}
        if we call $\sigma$ the $\text{softmax}$ function and $\widetilde E^L$ the powered, normalized adjancency matrix $\widetilde A^L$ after applying the explanation $\mathcal{E}_u$.
        Because $\text{softmax}$ is a Lipschitz continuous function with Lipschitz constant $1$ with respect to norm $\lVert \cdot \rVert_2$ \citep{gao_properties_2018} and considering induced matrix norm compatibility property \citep[Lemma 1.3.8.7]{van_de_geijn_advanced_2022},
        \begin{equation*}
            \begin{split}
                &\left\lVert \Psi(\mathcal{G}_{u})- \Psi(t(\mathcal{G}_{u}, \mathcal{E}_u))\right\rVert_2 \\
                =\ &\left\lVert \sigma\left[\left(\widetilde A^L X \Theta\right)_u\right] - \sigma\left[\left(\widetilde E^L X \Theta\right)_u\right] \right\rVert_2 \\
                \le\ &\left\lVert (X \Theta)^\intercal \right\rVert_2 \left\lVert \widetilde A_u^L - \widetilde E_u^L \right\rVert_2.
            \end{split}
        \end{equation*}
        Similarly, for a perturbation $\mathcal{G}_u'$ of $\mathcal{G}_u$ given by $\Sigma_{u}'$ being added to $X$,
        \begin{equation*}
            \begin{split}
                &\left\lVert \Psi(\mathcal{G}_{u}')- \Psi(t(\mathcal{G}_{u}', \mathcal{E}_u))\right\rVert_2 \\
                \le\ &\lVert [\left(X + \Sigma_{u}'\right) \Theta]^\intercal \rVert_2 \left\lVert \widetilde A_u^L - \widetilde E_u^L \right\rVert_2.
            \end{split}
        \end{equation*}
        Computing the mean over $\mathcal{K} \cup \{\mathcal{G}_u\}$:
        \begin{equation*}
            \begin{split}
                &\frac{1}{|\mathcal{K}| + 1} \sum_{\mathcal{G}_u' \in \mathcal{K} \cup \{\mathcal{G}_u\}} \left\lVert \Psi(\mathcal{G}_{u}')- \Psi(t(\mathcal{G}_{u}', \mathcal{E}_u))\right\rVert_2 \\
                \le\ &  \frac{1}{|\mathcal{K}| + 1} \Biggl[ \lVert (X \Theta)^\intercal \rVert_2 \left\lVert \widetilde A_u^L - \widetilde E_u^L \right\rVert_2 \\
                & \ \ \ \ \ \ \ \ \ \ \ \ \ \ \ \left. + \sum_{\mathcal{G}_u' \in \mathcal{K}}  \lVert [\left(X + \Sigma_{u}'\right) \Theta]^\intercal \rVert_2 \left\lVert \widetilde A_u^L - \widetilde E_u^L \right\rVert_2 \right] \\
                \le\ &  \lVert\Theta^\intercal\rVert_2 \left\lVert \widetilde A_u^L - \widetilde E_u^L \right\rVert_2  \frac{1}{|\mathcal{K}| + 1} \Biggl(\lVert X^\intercal \rVert_2 \\
                & \ \ \ \ \ \ \ \ \ \ \ \ \ \ \ \ \ \ \ \ \ \ \ \ \ \ \ \ \ \ \ \ \ \ \ \ \ \ \ + \left.\sum_{\mathcal{G}_u' \in \mathcal{K}} \lVert (X + \Sigma_{u}')^\intercal \rVert_2 \right) \\
                \le\ &  \gamma_\Theta \left\lVert \underset{\mathcal{E}_u}{\Delta} \widetilde A_u^L \right\rVert_2 \max \left(\{\lVert X^\intercal \rVert_2\} \cup \{\lVert (X + \Sigma_{u}')^\intercal \rVert_2\}_{\mathcal{G}_u' \in \mathcal{K}}\right) \\
                =\ &  \gamma_{\Theta, X, \Sigma} \left\lVert \underset{\mathcal{E}_u}{\Delta} \widetilde A_u^L \right\rVert_2.
            \end{split}
        \end{equation*}
        
        When $\Sigma$ is limited, the constant $\gamma$ may not depend on $\Sigma$.
    \end{proof}
    
    \begin{proof}[Proof of Theorem~\ref{theo:bound_unfaithfulness_2}]
        We know that
        \begin{equation*}
            \begin{split}
                &\left\lVert \Phi(\mathcal{G}_{u}')- \Phi(t(\mathcal{G}_{u}', \mathcal{E}_u))\right\rVert_2 \\
                =\ &\lVert \Phi(\mathcal{G}_{u}') - \Psi(\mathcal{G}_{u}') + \Psi(\mathcal{G}_{u}') - \Psi(t(\mathcal{G}_{u}', \mathcal{E}_u)) \\
                &+ \Psi(t(\mathcal{G}_{u}', \mathcal{E}_u)) - \Phi(t(\mathcal{G}_{u}', \mathcal{E}_u))\rVert_2.
            \end{split}
        \end{equation*}
        So, by using triangle inequality and \autoref{theo:bound_unfaithfulness},
        \begin{equation*}
            \begin{split}
                &\frac{1}{|\mathcal{K}| + 1} \sum_{\mathcal{G}_u' \in \mathcal{K} \cup \{\mathcal{G}_u\}} \left\lVert \Phi(\mathcal{G}_{u}')- \Phi(t(\mathcal{G}_{u}', \mathcal{E}_u))\right\rVert_2 \\
                \le\ & \frac{1}{|\mathcal{K}| + 1} \sum_{\mathcal{G}_u' \in \mathcal{K} \cup \{\mathcal{G}_u\}} \left\lVert \Phi(\mathcal{G}_{u}') - \Psi(\mathcal{G}_{u}')\right\rVert_2 \\
                &+ \frac{1}{|\mathcal{K}| + 1} \sum_{\mathcal{G}_u' \in \mathcal{K} \cup \{\mathcal{G}_u\}} \left\lVert \Psi(\mathcal{G}_{u}') - \Psi(t(\mathcal{G}_{u}', \mathcal{E}_u))\right\rVert_2 \\
                &+ \frac{1}{|\mathcal{K}| + 1} \sum_{\mathcal{G}_u' \in \mathcal{K} \cup \{\mathcal{G}_u\}} \left\lVert \Psi(t(\mathcal{G}_{u}', \mathcal{E}_u)) - \Phi(t(\mathcal{G}_{u}', \mathcal{E}_u))\right\rVert_2 \\
                \le\ & \gamma \left\lVert \underset{\mathcal{E}_u}{\Delta} \widetilde A_u^L \right\rVert_2 + 2\alpha.
            \end{split}
        \end{equation*}
    \end{proof}
    
    \begin{proof}[Proof of Lemma~\ref{theo:prob_bound_unfaithfulness}]
        From the proof of Lemma~\ref{theo:bound_unfaithfulness}, we know that
        \begin{equation*}
            \begin{split}
                &\frac{1}{|\mathcal{K}| + 1} \sum_{\mathcal{G}_u' \in \mathcal{K} \cup \{\mathcal{G}_u\}} \left\lVert \Psi(\mathcal{G}_{u}')- \Psi(t(\mathcal{G}_{u}', \mathcal{E}_u))\right\rVert_2 \le\ \\
                 &  \lVert\Theta^\intercal\rVert_2 \left\lVert \underset{\mathcal{E}_u}{\Delta} \widetilde A_u^L \right\rVert_2  \frac{1}{|\mathcal{K}| + 1} (\lVert X^\intercal \rVert_2 + \sum_{\mathcal{G}_u' \in \mathcal{K}} \lVert (X + \Sigma_{u}')^\intercal \rVert_2 )
            \end{split}
        \end{equation*}
        By using triangle inequality, we can write
        \begin{equation*}
            \begin{split}
                &\frac{1}{|\mathcal{K}| + 1} \sum_{\mathcal{G}_u' \in \mathcal{K} \cup \{\mathcal{G}_u\}} \left\lVert \Psi(\mathcal{G}_{u}')- \Psi(t(\mathcal{G}_{u}', \mathcal{E}_u))\right\rVert_2 \\
                \le\ &  \lVert\Theta^\intercal\rVert_2 \left\lVert \underset{\mathcal{E}_u}{\Delta} \widetilde A_u^L \right\rVert_2 \lVert X^\intercal \rVert_2 \\
                &+ \lVert\Theta^\intercal\rVert_2 \left\lVert \underset{\mathcal{E}_u}{\Delta} \widetilde A_u^L \right\rVert_2 \frac{1}{|\mathcal{K}| + 1} \sum_{\mathcal{G}_u' \in \mathcal{K}} \lVert (\Sigma_{u}')^\intercal \rVert_2.
            \end{split}
        \end{equation*}
        We can work in the perturbation summand:
        \begin{equation*}
            \begin{split}
                \sum_{\mathcal{G}_u' \in \mathcal{K}} \lVert (\Sigma_{u}')^\intercal \rVert_2 \le\ & \sum_{\mathcal{G}_u' \in \mathcal{K}} \lVert \Sigma_{u}' \rVert_\text{F} \\
                =\ & \sum_{\mathcal{G}_u' \in \mathcal{K}} \sqrt{\sum_{i=1}^{nd} \epsilon_{u',i}^2} \\
                =\ & \sigma \sum_{\mathcal{G}_u' \in \mathcal{K}} \sqrt{\sum_{i=1}^{nd} Z_{u',i}^2} \\
                \le\ & \sigma \sum_{\mathcal{G}_u' \in \mathcal{K}} \max\left(1, \ {\sum_{i=1}^{nd} Z_{u',i}^2}\right) \\
                \le\ & \sigma \sum_{\mathcal{G}_u' \in \mathcal{K}} \left(1 + {\sum_{i=1}^{nd} Z_{u',i}^2}\right) \\
                =\ & \sigma |\mathcal{K}| + \sigma Q,
            \end{split}
        \end{equation*}
        with
        \begin{itemize}
            \item $\epsilon_{j, i}~\sim~\mathcal{N}(0, \sigma^2)$;
            \item $Z_{j, i}~\sim~\mathcal{N}(0, 1)$;
            \item $Q~\sim~\chi_{|\mathcal{K}|nd}^2$;
            \item $(n, d)~=~\text{dim}(X)$.
        \end{itemize}
        
        If $\alpha_p$ is the $p$-percentile of $Q$, then the probability of the last inequality is $p$:
        \begin{equation*}
            \begin{split}
                &\frac{1}{|\mathcal{K}| + 1} \sum_{\mathcal{G}_u' \in \mathcal{K} \cup \{\mathcal{G}_u\}} \left\lVert \Psi(\mathcal{G}_{u}')- \Psi(t(\mathcal{G}_{u}', \mathcal{E}_u))\right\rVert_2 \\
                \le\ &  \lVert\Theta^\intercal\rVert_2 \left\lVert \underset{\mathcal{E}_u}{\Delta} \widetilde A_u^L \right\rVert_2 \lVert X^\intercal \rVert_2 \\
                &+ \lVert\Theta^\intercal\rVert_2 \left\lVert \underset{\mathcal{E}_u}{\Delta} \widetilde A_u^L \right\rVert_2 \frac{1}{|\mathcal{K}| + 1} \sum_{\mathcal{G}_u' \in \mathcal{K}} \lVert (\Sigma_{u}')^\intercal \rVert_2 \\
                \le\ & \gamma_1 \left\lVert \underset{\mathcal{E}_u}{\Delta} \widetilde A_u^L \right\rVert_2 + \lVert\Theta^\intercal\rVert_2 \left\lVert \underset{\mathcal{E}_u}{\Delta} \widetilde A_u^L \right\rVert_2 \frac{1}{|\mathcal{K}| + 1} \sigma (Q + |\mathcal{K}|) \\
                \le\ & \gamma_1 \left\lVert \underset{\mathcal{E}_u}{\Delta} \widetilde A_u^L \right\rVert_2 + \lVert\Theta^\intercal\rVert_2 \left\lVert \underset{\mathcal{E}_u}{\Delta} \widetilde A_u^L \right\rVert_2 \frac{1}{|\mathcal{K}| + 1} \sigma (\alpha_p + |\mathcal{K}|) \\
                =\ &   \gamma_1 \left\lVert \underset{\mathcal{E}_u}{\Delta} \widetilde A_u^L \right\rVert_2 + \gamma_2 \left\lVert \underset{\mathcal{E}_u}{\Delta} \widetilde A_u^L \right\rVert_2 \sigma (\alpha_p + |\mathcal{K}|)
            \end{split}
        \end{equation*}
        Notice that, when $\sigma$ approaches zero, the bound's probabilistic characteristic  becomes negligible.
        
        Finally, if we ask the bound to be $\xi$, then the probability is
        \[p = F_{\chi_{|\mathcal{K}|nd}^2}\left(\frac{\xi - \gamma_1 \left\lVert \underset{\mathcal{E}_u}{\Delta} \widetilde A_u^L \right\rVert_2}{\gamma_2 \left\lVert \underset{\mathcal{E}_u}{\Delta} \widetilde A_u^L \right\rVert_2 \sigma} - |\mathcal{K}| \right),\]
        $F_{\chi_{|\mathcal{K}|nd}^2}$ being the c.d.f. of $Q$.
    \end{proof}
    
    \begin{proof}[Proof of \autoref{theo:convexity}]
        The objective function of the problem in \autoref{eq:e2} can be written as
        \begin{equation*}
            \begin{split}
                f(\mathcal{E}) =\ & \left\lVert \widetilde A_i^L \text{diag}(\mathcal{E}) X \Theta - \widetilde A_i^L X \Theta\right\rVert_2^2 \\
                =\ &\left( \widetilde A_i^L \text{diag}(\mathcal{E}) X \Theta - \widetilde A_i^L X \Theta\right) \cdot \\
                &\ \cdot \left( \widetilde A_i^L \text{diag}(\mathcal{E}) X \Theta - \widetilde A_i^L X \Theta\right)^\intercal \\
                =\ &\mathcal{E}^\intercal \text{diag}\left[\left(\widetilde A_i^L\right)^\intercal\right] X \Theta \Theta^\intercal X^\intercal \text{diag}\left[\left(\widetilde A_i^L\right)^\intercal\right] \mathcal{E} \\
                &- 2 \mathcal{E}^\intercal \text{diag}\left[\left(\widetilde A_i^L\right)^\intercal\right] X \Theta \Theta^\intercal X^\intercal \left(\widetilde A_i^L\right)^\intercal \\
                &\ + \left\lVert\widetilde A_i^L X \Theta\right\rVert_2^2 \\
                =\ &\frac{1}{2}\mathcal{E}^\intercal Q \mathcal{E} + \mathcal{E}^\intercal c + \delta.
            \end{split}
        \end{equation*}
        Note also that \[Q = P^\intercal P \text{ with } \ P = \sqrt{2} \Theta^\intercal X^\intercal \text{diag}\left[\left(\widetilde A_i^L\right)^\intercal\right],\]
        thus $Q$ is symmetric and positive semidefinite. Since both the objective function and feasible set are convex, the optimization problem is also convex. 
    \end{proof}

\section{Datasets and implementation details}\label{append:implementation}
\label{sec:details_datasets}

\subsection{Datasets}

Bitcoin-Alpha and Bitcoin-OTC are real-world networks comprising 3783 and 5881 nodes (user accounts), respectively. Users rate their trust in each other using a score between $-10$ (total distrust) and $10$ (total trust). Then, user accounts are labeled as trusted or not based on how other users rate them. Accounts (nodes) have features such as average ratings. Ground-truth explanations for each node are provided by experts.
The synthetic datasets are available in \citep{GNNexplainer} and \citep{PGMExplainer}.
\autoref{tab:details_datasets} shows summary statistics for all datasets.

\begin{table}[!htb]
\centering
\caption{Statistics of the datasets used in our experiments.}
\adjustbox{width=0.35\textwidth}{
\begin{tabular}{lccc}
\toprule
\textbf{Dataset} & \textbf{nodes} & \textbf{edges} & \textbf{labels}\\
\midrule
BA-House & 700 & 4110 & 4\\
BA-Community & 1400 & 8920 & 8\\
BA-Grids & 1020 & 5080 & 2\\
Tree-Cycles & 871 & 1950 & 2\\
Tree-Grids & 1231 & 3410 & 2 \\
BA-Bottle & 700 & 3948 & 4\\
Bitcoin-Alpha & 3783 & 28248 & 2\\
Bitcoin-OTC & 5881 & 42984 & 2\\
\bottomrule
\end{tabular}}
\label{tab:details_datasets}
\end{table}

\subsection{Implementation details}
\begin{table*}[!t]
\centering
\caption{Performance (mean and standard deviation of accuracy) in node classification tasks for the models to be explained.}

\adjustbox{width=\textwidth}{
\begin{tabular}{lccccccccc}
\toprule
  %& syn1  &  syn2 & syn3 & syn4 & syn5 & syn6 & bitcoin-alpha & bitcoin-otc \\ \toprule
  & \textbf{BA-House} &\textbf{ BA-Community}  &\textbf{ BA-Grids} & \textbf{Tree-Cycles} & \textbf{Tree-Grids} & \textbf{ BA-Bottle} & \textbf{Bitcoin-Alpha} & \textbf{Bitcoin-OTC} \\ \toprule
   %&\multicolumn{7}{c}{GCN} \\ \hline

 GCN & $97.9 \pm{1.6}$  & $85.6\pm{1.7}$  & $99.9\pm{0.2}$ &  $97.8\pm{1.2}$  & $88.9\pm{1.8}$ &  $99.0\pm{0.1}$ &  $93.3\pm{0.4}$  & $93.2\pm{0.6}$ \\
  ARMA &  $98.1\pm{2.3}$ & $92.8 \pm{2.3}$& $99.5 \pm{0.6}$& $96.9 \pm{1.5}$&$92.4 \pm{2.3}$& $99.6 \pm{0.9}$& $93.6\pm{1.4}$ & $92.9\pm{1.1}$\\

 %GAT & $100.0 \pm{}$ & $77.8\pm{}$ & $100.0\pm{}$ & $97.7\pm{}$ & $93.4\pm{}$ &$98.5\pm{}$ & $89.6\pm{}$ & $89.6\pm{}$\\
 GATED & $98.0\pm{1.3}$ & $92.3\pm{2.7}$ & $99.9 \pm{0.2}$& $97.8\pm{2.2}$ & $94.4 \pm{3.0}$& $99.4\pm{0.6}$ & $94.4\pm{1.3}$& $93.8\pm{0.7}$\\
 GIN & $95.6 \pm{5.0}$  & $87.0 \pm{1.5}$ & $99.4\pm{0.6}$ &  $97.8 \pm{1.1}$ & $91.4 \pm{2.6}$ &  $98.8\pm{1.0}$ &  $93.4 \pm{1.0}$ & $92.6\pm{1.1}$ \\
 
\bottomrule
\end{tabular}}
\label{tab:classification}
\end{table*}

We ran all experiments using a laptop with an Intel Xeon 2.20 GHz CPU,  13 GB RAM DDR5, and RTX 3060ti 16GB GPU.

The architecture of the GNNs to be explained in this work are: 
$i$) a GCN model with 3 layers (20 hidden units) followed by a two-layer MLP with 60 hidden units; $ii$) an ARMA model with 3 layers (20 hidden units each) followed by a two-layer MLP with 60 hidden units; $iii$) a GIN model with 3 layers (20 hidden units each); and $iv$) a GATED model with 3 layers (100 hidden units) followed by a two-layer MLP with 300 hidden units. We train all these GNNs using with a learning rate of 0.001. All models use ReLU activations in their hidden units. We also use the one-hot vector of the degree as node features.

For computational reasons, on Bitcoin datasets, we provide results for 500 test nodes whenever the candidate explanation set contemplates 3-hop neighborhood (matching the GNNs we want to explain). For 1-hop cases, we use the full set of 2000 nodes as in the original setup.

\paragraph{Node and edge-level explanations.} DnX and FastDnX were originally designed to generate node-level explanations like PGM-Explainer but some baselines such as PGexplainer and GNNexplainer provide edge-level explanations. For a fair comparison with these baselines, we use a procedure to convert edge-level explanations to node-level ones (and vice-versa). To convert node scores to edge scores, we sum the scores of endpoints of each edge, i.e., an edge $(u,v)$ gets score $s_{u,v} = s_u + s_v$ where $s_u$ and $s_v$ are node scores. For the reverse, we say $s_u = {|\mathcal{N}_u|}^{-1}\sum_{j \in \mathcal{N}_u} s_{u, j}$ where $\mathcal{N}_u$ is the neighborhood of $u$.

 \section{Additional experiments}
\label{append:sup_experiments} 

\vspace{-1pt}
\paragraph{GNN performance.} \autoref{tab:classification} shows the classification accuracy of each model we explain in our experimental campaign. Means and standard deviations reflect the outcome of 10 repetitions. All classifiers achieve accuracy $>85\%$.
\begin{table}[!htb]
\centering
\caption{Distillation accuracy and elapsed time for ARMA, GATED and GIN. For all cases, accuracy $>88\%$ and learning $\Psi$ takes less than a minute.} 
\adjustbox{width=\columnwidth}{
\begin{tabular}{llcc}
\toprule
\textbf{Model} &\textbf{Dataset} & \textbf{Accuracy} & \textbf{Time (s)} \\ \midrule
%\multicolumn{3}{c}{\textbf{ARMA}} \\\midrule
%\textbf{Dataset} & \textbf{Accuracy} & \textbf{Time (s)} \\ \midrule
\multirow{8}{*}{\shortstack{ARMA}}

&BA-House & $97.7 \pm{0.01}$ & 7.254\\
&BA-Community  & $96.7\pm{0.03}$ & 22.848 \\
&BA-Grids & $100.0\pm{0.00}$ & 2.701 \\
&Tree-Cycles & $98.8\pm{0.03}$ & 6.331 \\
&Tree-Grids & $91.7\pm{0.15}$ & 4.152\\
&BA-Bottle & $100.0\pm{0.00}$ & 6.312\\
&Bitcoin-Alpha & $91.3\pm{0.01}$ & 28.613 \\
&Bitcoin-OTC & $91.2\pm{0.03}$ & 39.832\\

\midrule 
\multirow{8}{*}{\shortstack{GATED}}

%\multicolumn{3}{c}{\textbf{GATED}} \\\midrule
%\textbf{Dataset} & \textbf{Accuracy} & \textbf{Time (s)} \\ \midrule
&BA-House & $98.0 \pm{0.01}$ &  5.954\\
&BA-Community  & $92.1\pm{0.04}$ & 27.847 \\
&BA-Grids & $100.0\pm{0.00}$ & 2.679 \\
& Tree-Cycles & $99.4\pm{0.04}$ & 6.643\\
&Tree-Grids & $90.9\pm{0.09}$ & 4.609 \\
&BA-Bottle & $100.0\pm{0.00}$ & 6.171 \\
&Bitcoin-Alpha & $91.3\pm{0.02}$ & 26.611 \\
&Bitcoin-OTC & $91.1\pm{0.01}$ & 42.626\\ 
\midrule 

\multirow{8}{*}{\shortstack{GIN}}

%&\multicolumn{3}{c}{\textbf{GIN}} \\\midrule
%\textbf{Dataset} & \textbf{Accuracy} & \textbf{Time (s)} \\ \midrule
&BA-House & $93.4 \pm{0.46}$ & 11.461 \\
&BA-Community  & $90.1\pm{0.13}$ & 24.204 \\
&BA-Grids & $100.0\pm{0.00}$ & 3.884 \\
& Tree-Cycles & $95.8\pm{0.02}$ & 3.335\\
&Tree-Grids & $88.1\pm{0.06}$ & 5.198\\
&BA-Bottle & $100.0\pm{0.01}$ & 6.970\\
&Bitcoin-Alpha & $95.5\pm{0.03}$ & 44.359\\
&Bitcoin-OTC & $89.8\pm{2.04}$ & 55.860\\ 

\bottomrule
\end{tabular}
}
\label{tab:distillation_models}
\end{table}

\vspace{-1pt}
\paragraph{Distillation.} \autoref{tab:distillation_models} shows the extent to which the distiller network $\Psi$ agrees with the GNN $\Phi$. We measure agreement as accuracy, using the predictions of $\Phi$ as ground truth. The means and standard deviations reflect the outcome of 10 repetitions. Additionally, \autoref{tab:distillation_models} shows the  time elapsed during the distillation step. For all cases, we observe accuracy values over $88\%$.

\vspace{-1pt}
\paragraph{Results for edge-level explanations.} \autoref{tab:result_exp_edges} and \autoref{tab:result_exp_bitcoin_edges} complement Tables 1 and 2 in the main paper, showing results for edge-level explanations. All experiments were repeated ten times. For all datasets, we measure performance in terms of AUC, following \citet{PGExplainer}. Both tables corroborate our findings from the main paper. In most cases, FastDnX is the best or second-best model for all models and datasets. For GCN and GATED, PGExplainer yields the best results. Overall, both Dnx and FastDnX outperform GNNExplainer and PGM-Explainer. Remarkably FastDnX and DnX's performance is steady, with small fluctuations depending on the model we are explaining. The same is not true for the competing methods, e.g., PGExplainer loses over $15\%$ AUC for the BA-Community (cf., GCN and ARMA). Note also that FastDnX and DnX are the best models on the real-world datasets.

\begin{table*}[thb]
\centering
\caption{Performance (mean and standard deviation of AUC) of explainer models on synthetic datasets for edge-level explanations. \textcolor{royal}{Blue} and \textcolor{light}{Green} numbers denote the best and second-best methods, respectively.}
\adjustbox{width=\textwidth}{
\begin{tabular}{ccccccccc}
\toprule

 \textbf{Model} & \textbf{Explainer} & \textbf{BA-House} & \textbf{BA-Community}  & \textbf{BA-Grids} &\textbf{ Tree-Cycles} & \textbf{Tree-Grids} &  \textbf{BA-Bottle} \\ \toprule
 %GCN
\multirow{5}{*}{\shortstack{GCN}}

&GNNExplainer & $82.4\pm{6.0}$  & $71.1\pm{6.0}$ &  $80.9\pm{1.0}$ & $58.4\pm{1.0}$ & $53.9\pm{2.0}$ &$82.8\pm{1.0}$  \\
&PGExplainer & \textcolor{royal}{$\bm{99.9\pm{.01}}$}  & \textcolor{royal}{$\bm{99.9\pm{.01}}$} &  $94.1\pm{9.0}$ & \textcolor{royal}{$\bm{92.3\pm{5.0}}$} & \textcolor{royal}{$\bm{79.4\pm{2.0}}$} &\textcolor{royal}{$\bm{99.9\pm{.01}}$}  \\
&PGMExplainer & $56.2\pm{0.3}$ & $53.0\pm{0.4}$ & $68.9\pm{0.2}$ & $62.1\pm{0.2}$ & $66.3\pm{0.4}$ & $53.4\pm{0.5}$\\

\cmidrule{2-8}
&DnX & $95.7\pm{.09}$ & $87.5\pm{.09}$ & \textcolor{royal}{$\bm{97.3\pm{.03}}$} & $80.5\pm{.02}$ & $72.6\pm{0.1}$ & $94.5\pm{0.1}$ \\
&FastDnX & \textcolor{light}{$\bm{99.4\pm{\text{NA}}}$} & \textcolor{light}{$\bm{93.0\pm{\text{NA}}}$} & \textcolor{light}{$\bm{96.7\pm{\text{NA}}}$} & \textcolor{light}{$\bm{89.1\pm{\text{NA}}}$} & \textcolor{light}{$\bm{77.4\pm{\text{NA}}}$} & \textcolor{light}{$\bm{99.5\pm{\text{NA}}}$} \\ \midrule \midrule
 
\multirow{5}{*}{\shortstack{ARMA}}
 %ARMA
 & GNNExplainer & $80.8\pm{1.0}$  & $75.1\pm{1.0}$ &  $69.4\pm{0.1}$ & $59.2\pm{2.0}$ & $59.9\pm{1.0}$ & $85.8\pm{.01}$  \\
 & PGExplainer & $93.3\pm{1.0}$  & $82.7\pm{.01}$ &  $93.1\pm{.01}$ & \textcolor{royal}{$\bm{92.4\pm{2.0}}$} & \textcolor{royal}{$\bm{84.0\pm{.03}}$} & \textcolor{light}{$\bm{98.0\pm{.02}}$}  \\
 &PGMExplainer & $82.2\pm{0.2}$  & $49.6\pm{0.7}$  &  $40.4\pm{0.2}$ & $64.7\pm{0.2}$ & $69.0\pm{0.3}$ & $63.1\pm{0.3}$  \\

\cmidrule{2-8}
 & DnX & \textcolor{light}{$\bm{95.7\pm{.06}}$} & \textcolor{light}{$\bm{87.9\pm{.07}}$} & \textcolor{light}{$\bm{97.3\pm{.04}}$} & $80.4\pm{0.3}$ & $72.8\pm{0.1}$ & $94.4\pm{.08}$ \\
 & FastDnX & \textcolor{royal}{$\bm{99.7\pm{\text{NA}}}$} & \textcolor{royal}{$\bm{96.8\pm{\text{NA}}}$} & \textcolor{royal}{$\bm{97.9\pm{\text{NA}}}$} & \textcolor{light}{$\bm{87.1\pm{\text{NA}}}$} & \textcolor{light}{$\bm{78.4\pm{\text{NA}}}$} & \textcolor{royal}{$\bm{99.7\pm{\text{NA}}}$} \\ \midrule \midrule

\multirow{5}{*}{\shortstack{GATED}}

% GATED
&GNNExplainer & $79.1\pm{2.0}$  & $59.2\pm{3.0}$ &  $69.4\pm{2.0}$ & $62.6\pm{2.0}$ & $53.6\pm{1.0}$ &$73.9\pm{2.0}$  \\
&PGExplainer & \textcolor{royal}{$\bm{99.9\pm{2.0}}$}  & \textcolor{royal}{$\bm{ 99.1 \pm{2.0}}$} &  \textcolor{royal}{$\bm{99.8\pm{1.0}}$} & $76.5\pm{9.0}$ & \textcolor{royal}{$\bm{97.0\pm{3.0}}$} &\textcolor{royal}{$\bm{99.9\pm{0.1}}$}  \\
&PGMExplainer & $49.3\pm{0.4}$  & $47.2\pm{0.5}$  &  $46.2\pm{0.3}$ & $45.8\pm{0.6}$ & $53.1\pm{0.8}$ &$49.1\pm{0.7}$  \\

\cmidrule{2-8}
&DnX & $95.8\pm{.07}$ & $84.4\pm{.09}$ & \textcolor{light}{$\bm{97.3\pm{.02}}$} & \textcolor{light}{$\bm{81.1\pm{0.1}}$} & $72.8\pm{0.1}$ & $94.4\pm{.09}$ \\ 
&FastDnX & \textcolor{light}{$\bm{99.5\pm{\text{NA}}}$} & \textcolor{light}{$\bm{93.7\pm{\text{NA}}}$} & $96.6\pm{\text{NA}}$ & \textcolor{royal}{$\bm{89.2\pm{\text{NA}}}$} & \textcolor{light}{$\bm{78.6\pm{\text{NA}}}$} & \textcolor{light}{$\bm{99.6\pm{\text{NA}}}$} \\ \midrule \midrule

\multirow{3}{*}{\shortstack{GIN}}

  %GIN
% & GNNExplainer &  &  &   &  &  &  \\
% & PGExplainer &  &  &   &  &  &   \\
  & PGMExplainer & $52.4\pm{0.3}$ &  $52.0\pm{0.4}$ & $52.2\pm{1.0}$ &  $67.4\pm{0.2}$& $63.9\pm{0.2}$ & $53.3\pm{0.4}$ \\
  \cmidrule{2-8}
 & DnX &  \textcolor{light}{$\bm{95.7\pm{0.1}}$} & \textcolor{light}{$\bm{87.2\pm{0.1}}$} & \textcolor{royal}{$\bm{97.3\pm{0.1}}$} & \textcolor{light}{$\bm{81.8\pm{0.2}}$} & \textcolor{light}{$\bm{72.0\pm{0.1}}$} & \textcolor{light}{$\bm{94.2\pm{0.1}}$}\\
 
 & FastDnX & \textcolor{royal}{$\bm{97.9\pm{\text{NA}}}$} & \textcolor{royal}{$\bm{91.0\pm{\text{NA}}}$}& \textcolor{light}{$\bm{97.1\pm{\text{NA}}}$} &  \textcolor{royal}{$\bm{86.3\pm{\text{NA}}}$} & \textcolor{royal}{$\bm{72.3\pm{\text{NA}}}$} & \textcolor{royal}{$\bm{97.0\pm{\text{NA}}}$}\\

% syn1
% Accuracy:  0.9965
% Precision:  0.9965
% Auc:  0.979241758013846
% syn2
% Accuracy:  0.94725
% Precision:  0.94725
% Auc:  0.9106418121195393
% syn3
% Accuracy:  0.9401234567901234
% Precision:  0.948171206225681
% Auc:  0.9717303349361439
% syn4
% Accuracy:  0.7847222222222222
% Precision:  0.7847222222222222
% Auc:  0.8634922969150769
% syn5
% Accuracy:  0.7748456790123457
% Precision:  0.7773649171698406
% Auc:  0.7230109768658525
% syn6
% Accuracy:  0.9915
% Precision:  0.9915
% Auc:  0.9705748969506456

\bottomrule
\end{tabular}}
%\end{adjustbox}
\label{tab:result_exp_edges}
\end{table*}

\begin{table}[p]
\centering
\caption{Performance (AUC) of edge-level explanations for real-world datasets, measured over 500 test nodes.}
%\adjustbox{width=0.48\textwidth}
{
\begin{tabular}{cl|c|c}
\toprule
%&\multicolumn{1}{c}{} &\multicolumn{1}{c}{Bitcoin-Alpha} & \multicolumn{1}{c}{Bitcoin-OTC}\\
\textbf{Model} & \textbf{Explainer} &  \textbf{Bitcoin-Alpha}  &  \textbf{Bitcoin-OTC}\\ \hline
\multirow{5}{*}{\shortstack{GCN \\ (3-hop)}}

& GNNex  &  $94.0$ & $97.3$ \\
& PGEx   & $59.7$ & $54.2$ \\
& PGMEx  & $75.8$ & $52.7$ \\
& DnX  & \textcolor{light}{$\bm{97.1}$} &\textcolor{light}{$\bm{98.1}$}\\
& FastDnX  & \textcolor{royal}{$\bm{97.3}$} & \textcolor{royal}{$\bm{99.1}$} \\

\midrule \midrule 
\multirow{5}{*}{\shortstack{ARMA \\ (3-hop)}}
& GNNex  &  $79.6$ & $92.3$ \\
& PGEx   & $47.8$ & $45.4$ \\
& PGMEx  & $76.2$ & $73.8$ \\
& DnX  & \textcolor{light}{$\bm{96.4}$} &\textcolor{light}{$\bm{98.7}$} \\
& FastDnX  & \textcolor{royal}{$\bm{97.4}$} &\textcolor{royal}{$\bm{99.1}$} \\

\midrule \midrule
\multirow{5}{*}{\shortstack{GATED \\ (3-hop)}}

& GNNex  & $92.5$ & $93.6$ \\
& PGEx   & $34.7$ & $33.7$ \\
& PGMEx  & $86.3$ & $83.1$ \\
& DnX  & \textcolor{light}{$\bm{96.3}$} & \textcolor{light}{$\bm{98.6}$} \\
& FastDnX  & \textcolor{royal}{$\bm{97.2}$} & \textcolor{royal}{$\bm{99.1}$} \\

\midrule \midrule 
\multirow{3}{*}{\shortstack{GIN \\ (3-hop)}} 
% & GNNex  &  - & - \\
% & PGEx   & - & - \\
& PGMEx  & $52.8$ & $76.1$ \\
& DnX  &  \textcolor{light}{$\bm{96.4}$} &\textcolor{light}{$\bm{98.7}$} \\
& FastDnX  & \textcolor{royal}{$\bm{97.5}$} &\textcolor{royal}{$\bm{98.9}$} \\

\bottomrule
\end{tabular}}

\label{tab:result_exp_bitcoin_edges}
\end{table}

\vspace{-1pt}
\paragraph{More results for node-level explanations.}
\autoref{tab:result_exp_bitcoin_node_others} shows results for ARMA, GATED, and GIN. As observed in \autoref{tab:result_exp_bitcoin_node} for GCNs, DnX FastDnX are the best methods.

\begin{table}[htb]
\centering
\caption{Performance (average precision) of node-level explanations for real-world datasets, measured over 500 test nodes.}
\adjustbox{width=0.48\textwidth}{
\begin{tabular}{cl|ccc|ccc}
\toprule
&\multicolumn{1}{c}{} &\multicolumn{3}{c}{\textbf{Bitcoin-Alpha}} & \multicolumn{3}{c}{\textbf{Bitcoin-OTC}}\\
\midrule

\textbf{Model} & \textbf{Explainer} &    \textbf{top 3} & \textbf{top 4} &\textbf{ top 5}  &  \textbf{top 3} & \textbf{top 4} & \textbf{top 5}\\ \midrule

\multirow{5}{*}{\shortstack{ARMA \\ (3-hop)}} 
& GNNEx  &  $80.9$ & $78.9$ & $74.5$  & $73.2$  & $69.5$ & $64.0$\\
& PGEx   & $72.5$ &  $67.8$ & $65.0$ &  $69.7$ & $68.7$ & $61.45$ \\
& PGMEx  & $73.8$ & $66.4$ & $58.3$ & $69.1$ & $62.8$ &  $54.9$\\
\cmidrule{2-8}
& DnX  & \textcolor{royal}{$\bm{95.0}$} & \textcolor{royal}{$\bm{90.1}$} & \textcolor{royal}{$\bm{84.9}$} & \textcolor{royal}{$\bm{93.6}$} & \textcolor{royal}{$\bm{88.1}$} & \textcolor{royal}{$\bm{83.2}$} \\
& FastDnX  & \textcolor{light}{$\bm{90.2}$} & \textcolor{light}{$\bm85.8$} & \textcolor{light}{$\bm81.0$} & \textcolor{light}{$\bm{87.7}$} & \textcolor{light}{$\bm{83.3}$} & \textcolor{light}{$\bm{79.2}$} \\

\midrule \midrule 

\multirow{5}{*}{\shortstack{GATED \\ (3-hop)}} 
& GNNEx  &  $80.1$ & $75.1$ & $769.6$  & $75.9$  & $70.9$ & $66.0$\\
& PGEx   & $75.4$ &  $74.8$ & $67.8$ &  $72.5$ & $70.5$ & $65.1$ \\
& PGMEx  & $80.2$ & $76.5$ & $72.4$ & $77.5$ & $72.1$ &  $67.3$\\
\cmidrule{2-8}

& DnX  & \textcolor{royal}{$\bm{94.4}$}  & \textcolor{royal}{$\bm{89.6}$} & \textcolor{royal}{$\bm{84.4}$} & \textcolor{royal}{$\bm{93.4}$} & \textcolor{royal}{$\bm{88.9}$} & \textcolor{royal}{$\bm{83.5}$}\\
& FastDnX  & \textcolor{light}{$\bm{89.6}$} & \textcolor{light}{$\bm{85.3}$} & \textcolor{light}{$\bm{80.0}$} & \textcolor{light}{$\bm{88.0}$} & \textcolor{light}{$\bm{83.4}$} & \textcolor{light}{$\bm{79.0}$} \\

\midrule \midrule
\multirow{3}{*}{\shortstack{GIN \\ (3-hop)}} 
& PGM-Ex  & $58.7$ & $49.2$ & $40.8$ & $57.6$ & $47.8$ &  $40.2$\\
\cmidrule{2-8}

& DnX  & \textcolor{royal}{$\bm{94.3}$} & \textcolor{royal}{$\bm{88.9}$} & \textcolor{royal}{$\bm{83.2}$} & \textcolor{royal}{$\bm{94.0}$} & \textcolor{royal}{$\bm{89.2}$} & \textcolor{royal}{$\bm{83.6}$} \\
& FastDnX  & \textcolor{light}{$\bm{85.0}$} & \textcolor{light}{$\bm{80.4}$} & \textcolor{light}{$\bm{74.6}$} & \textcolor{light}{$\bm{82.6}$} & \textcolor{light}{$\bm{77.1}$} & \textcolor{light}{$\bm{71.3}$} \\

\bottomrule
\end{tabular}}

\label{tab:result_exp_bitcoin_node_others}
\end{table}

\end{document}